\documentclass[10pt,journal,cspaper,compsoc]{IEEEtran}
\usepackage[colorlinks=false]{hyperref}
\usepackage{algorithm}
\usepackage{amsthm}
\usepackage{amsfonts}
\usepackage{amssymb}
\usepackage{algorithm}
\usepackage{algorithmic}
\usepackage{amsmath}
\usepackage{graphicx}
\usepackage{bm}
\usepackage{booktabs}
\usepackage{multirow}
\usepackage{subfig}
\newtheorem{theorem}{Theorem}
\usepackage{float}

\DeclareMathAlphabet{\mathsfsl}{OT1}{cmss}{m}{sl}
\newcommand*{\VEC}[1]{\ensuremath{\boldsymbol{#1}}}		
\newcommand*{\MATRIX}[1]{\ensuremath{\mathsfsl{#1}}}

\newcommand*{\realR}{\mathbb{R}}

\newcommand*{\TheFig}[1]{Fig.\ref{#1}}
\newcommand*{\TheTable}[1]{Tab.\ref{#1}}

\newcommand{\MY}{\MATRIX{Y}}
\newcommand{\ML}{\MATRIX{X}}
\newcommand{\MS}{\MATRIX{B}}
\newcommand{\MO}{\MATRIX{\Omega}}

\begin{document}
\title{Region-wise matching for image inpainting based on adaptive weighted low-rank decomposition}
\author{Shenghai~Liao,~\IEEEmembership{}
        Xuya~Liu,~\IEEEmembership{}
        Ruyi~Han,~\IEEEmembership{}
        Shujun~Fu,~\IEEEmembership{}
        Yuanfeng~Zhou,~\IEEEmembership{}
        and~Yuliang~Li~\IEEEmembership{}
\IEEEcompsocitemizethanks{\IEEEcompsocthanksitem S. Liao, R. Han, and S. Fu are with the School of Mathematics, Shandong University, Jinan 250100, China. E-mail: liaoshenghai@mail.sdu.edu.cn, 347915753@qq.com, shujunfu@163.com.
\IEEEcompsocthanksitem X. Liu is with the School of Computer Science and Technology, Shandong Jianzhu University, Jinan 250101, China. E-mail: xuya\_liu@hotmail.com.
\IEEEcompsocthanksitem Y. Zhou is with the School of Software, Shandong University, Jinan 250101, China. E-mail: yfzhou@sdu.edu.cn.
\IEEEcompsocthanksitem Y. Li is with the Department of Intervention Medicine, the Second Hospital of Shandong University, Jinan 250033, China. E-mail: lyl.pro@sdu.edu.cn.}
\thanks{Corresponding author: Shujun~Fu.}
}
\IEEEtitleabstractindextext{
\begin{abstract}
Digital image inpainting is an interpolation problem, inferring the content in the missing (unknown) region to agree with the known region data such that the interpolated result fulfills some prior knowledge.
Low-rank and nonlocal self-similarity are two important priors for image inpainting. Based on the nonlocal self-similarity assumption, an image is divided into overlapped square target patches (submatrices) and the similar patches of any target patch are reshaped as vectors and stacked into a patch matrix. Such a patch matrix usually enjoys a property of low rank or approximately low rank, and its missing entries are recoveried by low-rank matrix approximation (LRMA) algorithms. Traditionally, $n$ nearest neighbor similar patches are searched within a local window centered at a target patch. However, for an image with missing lines, the generated patch matrix is prone to having entirely-missing rows such that the downstream low-rank model fails to reconstruct it well. To address this problem, we propose a region-wise matching (RwM) algorithm by dividing the neighborhood of a target patch into multiple subregions and then search the most similar one within each subregion. A non-convex weighted low-rank decomposition (NC-WLRD) model for LRMA is also proposed to reconstruct all degraded patch matrices grouped by the proposed RwM algorithm. We solve the proposed NC-WLRD model by the alternating direction method of multipliers (ADMM) and analyze the convergence in detail. Numerous experiments on line inpainting (entire-row/column missing) demonstrate the superiority of our method over other competitive inpainting algorithms. Unlike other low-rank-based matrix completion methods and inpainting algorithms, the proposed model NC-WLRD is also effective for removing random-valued impulse noise and structural noise (stripes).
Extension to the restoration of remote sensing images demonstrate that the proposed model has competitive performance for destriping (stripe noise removal) and repairing dead lines.
\end{abstract}
}
\maketitle
\IEEEdisplaynontitleabstractindextext
\section{Introduction}
Digital images often suffer from missing information randomly or structurally (entire-row/column missing) during imaging and transmission. As an interpolation problem, image inpainting aims at restoring missing pixels accurately \cite{bertalmio2000image,guillemot2014image}. Image inpainting has been applied to the restoration of old photographs \cite{bertalmio2000image,YU2021590}, the removal of scratch and blotch in archived films \cite{ji2011robust,li2013detection,Herling2014}, and the restoration of remote sensing images (stripes, dead lines), etc. The main challenge of image inpainting is how to propagate information from the known elements to unknown elements under given coherent conditions (priors). The most commonly used priors follow the assumptions of smoothness \cite{shen2002mathematical, wali2019new}, nonlocal self-similarity \cite{criminisi2004region,Newson2017,Xu2022}, sparsity \cite{mairal2007sparse,Zhou2012BPFA}, and low-rank \cite{Lu2016IRNN,guo2017patch,gu2017weighted,Zheng2020MADC}.

Diffusion-based methods mainly take advantage of the intrinsic spatial coherence of geometrical structures and diffuse information from observed data to the unknown region driven by partial differential equations. Exemplar-based methods directly copy and paste the nonlocal similar patch to the corrupted region, while maintaining coherence with nearby pixels.
Sparse representation also received great attention in the literature and has achieved a lot of success such as K-SVD \cite{mairal2007sparse}. However, it is required to solve a large-scale optimization problem in dictionary learning stage and hence the overall computational complexity is high.

Recently, the convolutional neural networks (CNNs) have been successfully applied to a variety of low level computer vision problems \cite{yu2018generative, ulyanov2018deep, nazeri2019edgeconnect}.
Due to the powerful learning ability of CNNs, CNNs based approaches often perform well in specific tasks, however, millions of parameters are required to optimized with a large number of training images. So it is time consuming and usually needs a modern accelerating device such as GPU (Graphics Processing Unit). CNNs methods also lack interpretation and the performances drop evidently when test images or shapes of missing regions
deviate from the distribution in the training set.

In this paper, we revisit low-rank prior and nonlocal self-similarity for image inpainting.
Low-rank matrix approximation (LRMA) has been well investigated for decades and has extensive applications in machine learning and image processing. Given specific data fidelity requirement, the intuition is to approximate the observed matrix (often of large rank due to data corruption or noise) with another one that has a smaller rank. Low rank matrix factorization \cite{Shang2018,Xu2017} and rank minimization \cite{cai2010singular,Mohan2012} are two main branches of low-rank matrix reconstruction. The former approach factorizes the observed matrix into product of two smaller matrices, the latter one reconstructs the data matrix by imposing additional rank constraint upon the estimated matrix.

Since direct rank minimization problem is NP-hard, numerous algorithms instead adopt Nuclear Norm Minimization (NNM) as a relaxation. The nuclear norm of a matrix is sum of all singular values and it has been proved to be the tightest convex relaxation of the rank minimization problem \cite{candes2009exact}.
Based on NNM, Cai \cite{cai2010singular} proposed a singular value thresholding (SVT) algorithm to iteratively approximate the clean data matrix. Researchers also attempt to substitute nuclear norm by different types of non-convex relaxation \cite{lu2014generalized,Zheng2020MADC}. A truncated nuclear norm regularization (TNNR) was proposed for matrix completion \cite{zhang2012matrix}.
Gu \cite{gu2017weighted} proposed weighted nuclear norm minimization (WNNM) that assigned weight to each singular value.
Xie \cite{Xie2016} generalized the NNM to the Schatten p-norm
minimization with weights assigned to different singular values. Guo \cite{guo2017patch} utilized low rank approximation with truncated singular values to derive a closed-form estimate for each patch
matrix and designed a two-stage low rank approximation (TSLRA) scheme for image inpainting. Nie \cite{Nie2019NCMC} used logarithm function directly on all singular values. Other types of non-convex relaxation functions were summarized in \cite{lu2014generalized,Chen2020tip}.

LRMA is essentially an optimization of singular values \cite{lu2014generalized,gu2017weighted,Xie2016,Nie2019NCMC,Chen2020tip}, with specific regularization functions. However, the regularization function of a LRMA model is usually fixed \cite{Nie2019NCMC}, or relies on one or more hyper-parameters \cite{gu2017weighted,Chen2020tip} which are difficult to tune manually. Considering that the distribution of singular values of an image or a patch matrix has various possibilities, regularization on singular values should be adaptive enough to obtain better performance. Based on this point of view, we propose a non-convex regularization function which is adaptive to the distribution singular values. In section \ref{section-exps}, various experiments are conducted to demonstrate the advantage of the proposed adaptive non-convex model against other state-of-the-art methods.

Another shorthand of low-rank based inpainting methods is that they suffer from difficulties of handling structurally incomplete data matrix such as entire missing rows/columns.
For example, WNNM \cite{gu2017weighted} and TSLRA \cite{guo2017patch} both group similar patches and apply low-rank model to reconstruct all patch matrices, but still fail to repair images with entire missing rows/columns.
We analyse this problem and find that the way of searching similar patches is the main cause. Traditionally, $n$ nearest neighbor patches are searched within a local window centered at a target patch. However, for an image with missing lines, the generated patch matrix is prone to having entirely-missing rows such that the downstream low-rank model fails to reconstruct it well. To solve this limitation, we propose a region-wise matching (RwM) algorithm by dividing the neighborhood of a target patch into multiple subregions and then search the most similar one within each subregion. Two types of neighborhood division are discussed, i.e. grids and fan-shaped sectors.

With the proposed search strategy RwM, we propose an inpainting algorithm by applying the adaptive non-convex low rank decomposition model to complete each patch matrix grouped by RwM.
Benefits from the search strategy RwM, the proposed inpainting algorithm is robust in restoring missing lines and achieves the highest qualitative assessment scores on both peak signal-to-noise ratio (PSNR) and structural similarity (SSIM) metrics. While traditional low-rank based methods \cite{guo2017patch, gu2017weighted} fail to fill all missing lines completely, the proposed method removes both horizontal and vertical lines perfectly.

The main contributions of this paper are summarized as follows: 1) An adaptive non-convex weighted low rank decomposition model (NC-WLRD) is proposed for matrix completion, experiments on real images demonstrate the superiority of NC-WLRD over other matrix completion methods. The proposed NC-WLRD also obtains competitive performance on destriping of remote sensing images. 2) We propose a region-wise matching method for grouping similar patches and apply it to line inpainting. Numerous experiments show that the proposed image inpainting algorithm is more robust and superior compared to previous state-of-the-art inpainting algorithms, especially for completing entirely-missing rows/columns.

The rest of this paper is organized as follows. Section \ref{section2} briefly introduces some related works. Section \ref{section3} proposes our adaptive non-convex weighted low rank decomposition (NC-WLRD) model for matrix completion. We introduce a region-wise matching strategy in section \ref{section-ms-method}. In Section \ref{section-exps}, we conduct experiments to show the superiority of the proposed RwM-WLRD.
We conclude the paper in section \ref{section-conclusions}.
\section{Related works}
\label{section2}
In this section, we briefly present some image inpainting algorithms and robust principal component analysis.
\subsection{Image inpainting}
As mentioned in previous section, inpainting is an ill-posed inverse problem and thus proper prior constraint is necessary to achieve pleasant visual quality. In general, the image inpainting problem can be formulated as
\begin{equation}\label{eq-prior-min}
  \hat{\MATRIX{X}}=\mathop{\arg\min}_{\MATRIX{X}} prior(\MATRIX{X}),\ s.t.\ \mathcal{P}_{\MO}(\MATRIX{X})=\mathcal{P}_{\MO}(\MY)
\end{equation}
where $\mathcal{P}_{\MO}$ is a sampling operator in the observed region $\MO$. The constraint $\mathcal{P}_{\MO}(\MATRIX{X})=\mathcal{P}_{\MO}(\MY)$ means that the estimated image $\hat{\MATRIX{X}}$ should be coherent with the input $\MY$ in the observed region $\MO$. The objective $prior$ can be roughly categorized into several types: diffusion (total variation minimization), exemplar (minimum difference with observed information), sparsity (minimum number of atoms), and low-rank prior (rank minimization).

Diffusion-based algorithms fill a target pixel by propagating its surrounding information to the target driven by local smoothness, which can be regarded as a linear combination of neighbor pixels but with dynamic weights for different target pixels. The pattern of diffusion can be divided into isotropic and anisotropic \cite{tschumperle2006fast}. On the contrary, the exemplar-based methods \cite{criminisi2004region, Barnes2009PatchMatch} broadcast the whole patch into the target region.

Sparsity and low rank are another two classical priors which often achieve state-of-the-art in the field of image inpainting. These are usually combined with nonlocal self-similarity to achieve better performance. The sparsity-based methods rely on sparse representation. A main problem about sparse representation is ``atoms''. Many works adopt patch as the basis unit \cite{mairal2007sparse, shen2009image} for representation. However, the resulted coding coefficients may not be accurate since each patch is processed independently during dictionary learning and sparse coding. To address this problem, \cite{zhang2014group} utilized patch group as the basic unit instead of a single patch.
Low-rank matrix approximation (LRMA)-based image restoration models approximate the corrupted data matrix with another matrix that has a lower rank. In recent years, LRMA has been applied in various image restoration tasks \cite{dong2012nonlocal, Hu2013tnnr,JI2017410,zhang2019low, chen2020robust}.
\subsection{Robust principal component analysis}
Robust principal component analysis (RPCA) has been well investigated for decades \cite{wright2009robust, hsu2011robust, candes2011robust}.
Given an observed data matrix $\MY$, suppose that the underlying matrix is low-rank and corrupted by additive sparse noise, then $\MY$ is modeled as
\begin{equation}\label{eq-LS}
  \MY=\ML+\MS
\end{equation}
where $\ML$ is  the low-rank component and $\MS$ represents the sparse component. To obtain such decomposition, the RPCA \cite{lin2010augmented} solves the following minimization problem:
\begin{equation}\label{eq-nnm-rpca}
  \min_{\ML,\MS}\|\ML\|_{*}+\lambda\|\MS\|_{1}\ \ s.t.\ \ \ML + \MS=\MY
\end{equation}
where $\|\ML\|_{*}$ is the nuclear norm of $\ML$, defined as sum of all singular values of $\ML$,  and $\lambda$ is a balanced scalar parameter. Works in \cite{lin2010augmented} also compared several optimization algorithms for solving \eqref{eq-nnm-rpca}, and suggested inexact augmented Lagrange multiplier (IALM) since its higher accuracy, faster convergence, and easier implementation. Though it has many success in matrix completion, and even later works extend the regularization to more general case \cite{Wen2020RPCA}, applications of RPCA in image inpainting still remain at ``matrix completion'' perspective which restrict its application.

\section{Proposed model for matrix completion}
\label{section3}
\newcommand{\knn}{\|\ML\|_{\varphi,*}}
\newcommand{\shrinkage}{\mathcal{S}}
\newcommand{\sign}[1]{\mathrm{sgn}(#1)}
In this section we describe the matrix completion model based on non-convex weighted low rank decomposition (NC-WLRD), and show optimization details using  alternating direction method of multipliers (ADMM).
\subsection{Notation}
We denote $\MY\in\realR^{m\times n}$ as an observed data matrix, and the observed entries are indicated in a corresponding binary matrix $\MO$, e.g. $\MO_{ij}=1$ if $\MY_{ij}$ is observed and the unknown elements are marked zeros in $\MO$. Let $\ML$ be the underlying ground-truth matrix, we have $\MY=\ML\odot \MO$ where $\odot$ is element-wise multiplication. The $i$-th singular value of a matrix $\MATRIX{X}$ is denoted as $\sigma_{i}(\MATRIX{X})$, in a descending order, i.e. $\sigma_1(\MATRIX{X})$ is the largest singular value.  The inner product of two matrices $\MATRIX{X},\MATRIX{Y}$ is defined as $\langle \MATRIX{X},\MATRIX{Y} \rangle = tr(\MATRIX{X}^{T}\MATRIX{Y})$ where $tr(\cdot)$ is the trace operator. For any matrix $\MATRIX{X}\in \realR^{m\times n}$, its vectorization is defined as a column vector $\MATRIX{X}(:)\in\realR^{mn}$ by stacking all columns of $\MATRIX{X} $ vertically. For multichannel image patch $\MATRIX{X}\in \realR^{m\times n\times c}$, its vectorization is formed by vectorizing each channel separately and then stacking all vectors as a single vector.

For simplicity, we define kernel nuclear norm based on a nonnegative function $\varphi$ as
\begin{equation}\label{eq-knndef}
  \|\MATRIX{X}\|_{\varphi,*}=\sum_{i}\varphi(\sigma_{i}(\MATRIX{X}))
\end{equation}
We also introduce element-wised soft shrinkage operator $\shrinkage_{\bm{\epsilon}}(\bm{x})= \sign{\bm{x}}\cdot \max(|\bm{x}|-\bm{\epsilon},0)$ where $\mathrm{sgn}$ is the sign function.
\subsection{Proposed NC-WLRD model}
As discussed above, most low-rank based methods for image restoration involve a process of shrinking relevant singular values \cite{candes2009exact,gu2017weighted}. The shrinkage is achieved by introducing singular values of the designed matrix into loss function. Previous methods usually utilize a fixed loss function on singular values for all matrices in model design, and hence they are not adaptive enough. Furthermore, we observe that the largest singular value of $\MY$ may be lower than that of its clean version, and thus a shrinkage on all singular values is not expected.

To address this problem, we propose an adaptive non-convex weighted low-rank decomposition model for matrix completion as follows:
\begin{equation}\label{eq-model}
   \mathop{\arg\min}_{\ML,\MS}\knn+\lambda\|\MO\odot \MS\|_{1}~~s.t.~~\MY=\ML+\MS
\end{equation}
where $\lambda>0$ is a scalar, and $\lambda\|\MO\odot \MS\|_{1}$ is a weighted $l_1$ norm of the sparse component $\MS$, $\varphi$ is an adaptive piecewise smooth function:
\begin{equation}\label{eq-adaptive-scad}
  \varphi(\sigma) =
  \begin{cases}
    \sigma, & \mbox{if } \sigma\leq1 \\
    \frac{-\sigma^2+2\gamma\sigma-1}{2(\gamma-1)}, & \mbox{if } 1<\sigma \leq \gamma \\
    \frac{\gamma+1}{2}, & \mbox{otherwise}.
  \end{cases}
\end{equation}
The scalar $\gamma$ above is set adaptively according to the largest singular value of $\MY$: $\gamma = (\eta + \alpha)\cdot\sigma_{1}(\MY)$, $\eta>0$ is a parameter and $\alpha$ is the proportion of missing elements. Note that $\varphi$ is a non-convex function, and $\gamma$ is self-adaptive adjusted according to the missing rate $\alpha$ and the spectrum of the observed matrix $\MY$. The intuition behind this adaptive regularization originates from the observation that the distribution of singular values of an image has a long tail that progressively decrease towards zero, thus the model utilizes an identity function for small singular values less than 1. On the other hand, if the missing rate $\alpha$ is small, the distribution of large singular values of $\MY$ is highly correlated to that of the underlying clean matrix.

Given an observed data matrix $\MY\in\realR^{m\times n}$, model \eqref{eq-model} decomposes $\MY$ into a low-rank part $\ML$ and a sparse component $\MS$. For matrix completion and image inpainting, the low-rank estimation $\ML$ is the desired output. We will show in experiments that the proposed model is also effective for stripe noise removal (destriping), where the expected output is the sparse component $\MS$ for an image comprises of stripes has a lower rank (approximately rank one). A main advantage of the proposed model is that  it handles the matrix completion problem and noise removal (random-valued impulse noise and destriping) in a uniform framework.
\subsection{Optimization method}
We utilize the alternating direction method of multipliers (ADMM) \cite{Boyd2011} to solve the proposed model \eqref{eq-model} for its simplicity and convergence guarantee. The constraint optimization problem \eqref{eq-model} is first converted to its non-constraint form using the augmented Lagrange multipliers method. The augmented Lagrange function of \eqref{eq-model} is
\begin{equation}\label{eq-alm}
\begin{aligned}
  \Gamma(\ML,\MS,\MATRIX{A},\mu)&=\knn+\lambda\|\MO\odot\MS\|_{1}
       +\langle\MATRIX{A},\MY-\ML-\MS\rangle\\
       &+\frac{\mu}{2}\|\MY-\ML-\MS\|_{F}^{2}
       \end{aligned}
\end{equation}
where $\MATRIX{A}$ is a Lagrange multiplier matrix and $\mu>0$ is a penalty factor.
The minimization of \eqref{eq-alm} respects to $\ML,\MS$ is equivalent to the following except a constant term when $\MATRIX{A}$ is fixed:
\begin{equation}\label{eq-alm-2}
  \knn+\lambda\|\MO\odot\MS\|_{1}+\frac{\mu}{2}\| \frac{1}{\mu}\MATRIX{A}+\MY-\ML-\MS  \|_{F}^{2}.
\end{equation}
Then model \eqref{eq-model} is minimized iteratively as described in algorithm \ref{alg-admm-rpca}.
Given the k-th step's approximations $\ML^{(k)},\MATRIX{A}^{(k)}$, optimizing $\MS$ is straightforward, and $\MS^{(k+1)}$ is obtained via soft shrinkage:
\begin{equation}\label{eq-solve-S}
\begin{aligned}
  \MS^{(k+1)} &=\arg\min_{\MS}\lambda\|\MO\odot\MS\|_{1}
  +\frac{\mu^{(k)}}{2}\|\MATRIX{E}^{(k)}-\MS\|_{F}^{2}\\
  &= \MO\odot\shrinkage_{\bm{\epsilon}}(\MATRIX{E}^{(k)}) + (1-\MO)\odot \MATRIX{E}^{(k)},
\end{aligned}
\end{equation}
where $\MATRIX{E}^{(k)}=\frac{1}{\mu^{(k)}}\MATRIX{A}^{(k)}+\MY-\ML^{(k)}$ and $\bm{\epsilon}=\frac{\lambda}{\mu^{(k)}}\MO$.
\begin{algorithm}
\caption{NC-WLRD for Matrix Completion}
\label{alg-admm-rpca}
\hspace*{0.02in} {\bf Input: }
Observed data matrix $\MY$, indicator matrix $\MO$

\hspace*{0.02in} {\bf Initialization: }
$\ML^{(0)}=\MATRIX{0},\MS^{(0)}=\MATRIX{0},\MATRIX{A}^{(0)}=\|\MY\|_{2}^{-1}\cdot\MY$

\begin{algorithmic}[1]
\STATE {\bf set}$\ \mu^{(0)}>0,\rho>1,k=0$
\WHILE{not converged}

\STATE $\MATRIX{E}^{(k)}=\frac{1}{\mu^{(k)}}\MATRIX{A}^{(k)}+\MY-\ML^{(k)}$
    \STATE
       $\MS^{(k+1)}=\arg\min_{\MS}\lambda\|\MO\odot\MS\|_{1}
  +\frac{\mu^{(k)}}{2}\|\MATRIX{E}^{(k)}-\MS\|_{F}^{2}$

\STATE $\MATRIX{D}^{(k)}=\frac{1}{\mu^{(k)}}\MATRIX{A}^{(k)}+\MY-\MS^{(k+1)}$

  \STATE  $\ML^{(k+1)}=\arg\min_{\ML}\knn+
  \frac{\mu^{(k)}}{2}\|\MATRIX{D}^{(k)}-\ML\|_{F}^{2}$

\STATE $\MATRIX{A}^{(k+1)}=\MATRIX{A}^{(k)}+\mu^{(k)}(\MY-\MS^{(k+1)}-\ML^{(k+1)})$
\STATE  $\mu^{(k+1)}=\rho\mu^{(k)}$
\STATE $k=k+1$
\ENDWHILE

\end{algorithmic}
\hspace*{0.02in} {\bf Output}
$\ML=\ML^{(k)}, \MS=\MS^{(k)}$
\end{algorithm}

To update $\ML^{(k+1)}$, let $\MATRIX{D}^{(k)}=\frac{1}{\mu^{(k)}}\MATRIX{A}^{(k)}+\MY-\MS^{(k+1)}$, we have
\begin{equation}\label{eq-opti-L}
  \ML^{(k+1)}=\arg\min_{\ML}\knn+
  \frac{\mu^{(k)}}{2}\|\MATRIX{D}^{(k)}-\ML\|_{F}^{2}.
\end{equation}
\newcommand{\trace}[1]{tr(#1)}
Based on the property of Frobenius norm, we have
\begin{equation}\label{eq-opti-lowrank}
\begin{aligned}
  \ML^{(k+1)} & = \arg\min_{\ML}\knn+\frac{\mu^{(k)}}{2} \trace{(\MATRIX{D}^{(k)}-\ML)^T(\MATRIX{D}^{(k)}-\ML)} \\
  & = \arg\min_{\ML}\knn+\frac{\mu^{(k)}}{2} \trace{\ML^T\ML} - \mu^{(k)}\trace{\ML^T\MATRIX{D}^{(k)}} \\
  & = \arg\min_{\ML}\sum_{i=1}^{n}(\frac{\mu^{(k)}}{2}\sigma_i^2+\varphi(\sigma_i)) - \mu^{(k)}\trace{\ML^T\MATRIX{D}^{(k)}}
\end{aligned}
\end{equation}
Let $\ML=\MATRIX{U}\MATRIX{\Sigma}\MATRIX{V}^{T},\MATRIX{D}^{(k)}=\bar{\MATRIX{U}}\MATRIX{S}\bar{\MATRIX{V}}^{T}$ be the singular value decomposition of $\ML$ and $\MATRIX{D}^{(k)}$, $\MATRIX{\Sigma}=diag(\sigma_1,\cdots,\sigma_n),\MATRIX{S}=diag(s_1,\cdots,s_n)$. According to von Neumanns trace inequality \cite{mirsky1975trace}, we have
\begin{equation}
\trace{\ML^T\MATRIX{D}^{(k)}}\leq \sum_{i=1}^{n}\sigma_i s_i.
\end{equation}
$\trace{\ML^T\MATRIX{D}^{(k)}}$ reaches its upper bound if $\MATRIX{U}=\bar{\MATRIX{U}},\MATRIX{V}=\bar{\MATRIX{V}}$. Then it follows
\begin{equation*}
  \begin{aligned}
    &\arg\min_{\ML}\sum_{i=1}^{n}(\frac{\mu^{(k)}}{2}\sigma_i^2+\varphi(\sigma_i)) - \mu^{(k)}\trace{\ML^T\MATRIX{D}^{(k)}}\\
    \Leftrightarrow & \mathop{\arg\min}_{\sigma_{1}\geq \sigma_2\geq\cdots\geq\sigma_{n}\geq0}\sum_{i=1}^{n}(\frac{\mu^{(k)}}{2}\sigma_i^2+\varphi(\sigma_i)-\mu^{(k)}\sigma_i s_i ) \\
    \Leftrightarrow & \mathop{\arg\min}_{\sigma_1\geq\sigma_2\geq\cdots\geq\sigma_n\geq 0}\sum_{i=1}^{n}\frac{\mu^{(k)}}{2}(\sigma_i - s_i)^2 + \varphi(\sigma_i).
\end{aligned}
\end{equation*}
The above quadratic optimization problem is solvable for each $\sigma_i$, and we have the following optimal condition:
\begin{equation}\label{eq-solve-sigma}
\left\{
  \begin{aligned}
   &\mu^{(k)}(\sigma_i - s_i) + \varphi^{\prime}(\sigma_i) = 0,i=1,2,\cdots,n\\
   & \sigma_1\geq\sigma_2\geq\cdots\geq\sigma_n\geq 0
  \end{aligned}
  \right.
\end{equation}
A simple calculation gives us the optimal solution:
\begin{equation}\label{eq-shrinkage-sv-solution}
  \sigma_{i} = \begin{cases}
                 \max(\frac{\mu^{(k)}\cdot s_i-1}{\mu^{(k)}},0), & \mbox{if } \ s_i\leq \frac{\mu^{(k)}+1}{\mu^{(k)}} \\
                 \frac{\mu^{(k)}\cdot s_i-\frac{\gamma}{\gamma-1}}
                 {\mu^{(k)}-\frac{1}{\gamma-1}}, & \mbox{if } \frac{\mu^{(k)}+1}{\mu^{(k)}}<s_i<\gamma \\
                 s_i, & \mbox{otherwise}.
               \end{cases}
\end{equation}
Note that the above formulation of $\sigma_i$ has a positive factor on $s_i$ and hence the condition $\sigma_1\geq\sigma_2\geq\cdots\geq\sigma_n\geq 0$ is satisfied. The nonnegative of the second term in \eqref{eq-shrinkage-sv-solution} is easy to verified:
\begin{equation}
\begin{aligned}
&\frac{\mu^{(k)}+1}{\mu^{(k)}}<\gamma \Rightarrow \mu^{(k)}-\frac{1}{\gamma-1}>0;\\
&\frac{\mu^{(k)}+1}{\mu^{(k)}}<s_i \Rightarrow \mu^{(k)}\cdot s_i-\frac{\gamma}{\gamma-1}>\mu^{(k)}-\frac{1}{\gamma-1}>0.\\
\end{aligned}
\end{equation}
In summary, the updating of $\ML^{(k+1)}$ needs one SVD calculation of $\MATRIX{D}^{(k)}$, and shrinks the singular values according to formula \eqref{eq-shrinkage-sv-solution}. Namely, let
$\MATRIX{D}^{(k)}=\MATRIX{U}\MATRIX{S}\MATRIX{V}^{T}$ be the SVD of $\MATRIX{D}^{(k)}$, $\MATRIX{S}=diag(s_1,\cdots,s_n),\MATRIX{\Sigma}=diag(\sigma_1,\cdots,\sigma_n)$, then
\begin{equation}\label{Xk1}
  \ML^{(k+1)}=\MATRIX{U}\MATRIX{\Sigma}\MATRIX{V}^{T},\sigma_{i} = \begin{cases}
                 \max(\frac{\mu^{(k)}\cdot s_i-1}{\mu^{(k)}},0), & \mbox{if } \ s_i\leq \frac{\mu^{(k)}+1}{\mu^{(k)}} \\
                 \frac{\mu^{(k)}\cdot s_i-\frac{\gamma}{\gamma-1}}
                 {\mu^{(k)}-\frac{1}{\gamma-1}}, & \mbox{if } \frac{\mu^{(k)}+1}{\mu^{(k)}}<s_i<\gamma \\
                 s_i, & \mbox{otherwise}.
               \end{cases}
\end{equation}
We give the convergence analysis in \ref{appendix}.

\section{Region-wise matching for image inpainting}
\label{section-ms-method}
\subsection{Region-wise matching}
\newcommand{\patchsize}{\sqrt{m}\times \sqrt{m}}
\newcommand{\patchsizem}{\sqrt{m}\times \sqrt{m}}
\newcommand{\regionnumber}{{\tilde{n}}}
Nonlocal self-similarity (NSS) has been successfully applied in many image restoration tasks \cite{xu2015patch,li2019low}, and it usually brings performance improvements in most cases.
However, when contaminated by entire-row/column missing noise, most existing patch-based low-rank inpainting algorithms can not recovery the image well. The main reason is their adopted exhaustive block matching method that looks at all locations around a target patch to find patches as similar as possible. As a result, all patches within a group share the same locations of missing entries, and the constructed group matrix contains unobserved rows.

To remedy the above problem, we propose a region-wise matching (RwM) method to group similar patches.
Concretely, let $\bm{I}\in \realR^{H\times W}$ denote an image to be processed, where $H$ and $W$ stand for the height and width. For a patch of size $\patchsize$ at position $\VEC{p}=(x,y)^T$ (top-left corner), define $\bm{I}_{\VEC{p}}\in \realR^{\patchsize}$ as the corresponding data, and
\begin{equation}\label{eq-Rp}
  \bm{I}_{\VEC{p}} = \bm{R}_{\VEC{p}}(\bm{I}),
\end{equation}
where $\bm{R}_{\VEC{p}}(\cdot)$ is an operator that extracts the patch $\bm{I}_{\VEC{p}}$ from the image $\bm{I}$, and its transpose $\bm{R}_{\VEC{p}}^{T}(\cdot)$ produces an image with the same size as $\bm{I}$ by sending back a patch into position $\VEC{p}$ in the reconstructed image, padded with zeros elsewhere.

Then to fill each target patch $\bm{I}_{\VEC{p}}$, we search $\regionnumber$ similar patches located in $\regionnumber$ subregions around position $\VEC{p}$:
\begin{equation}\label{eq-N-patches}
  \bm{I}_{\VEC{q}_i} = \bm{R}_{\VEC{q}_i}(\bm{I}), \VEC{q}_i = \arg\min_{\VEC{q}\in \Gamma_{i}}\| \bm{R}_{\VEC{p}}(\bm{I}) - \bm{R}_{\VEC{q}}(\bm{I}) \|_F^2,
\end{equation}
where $ \Gamma_{i}$ is the set of points in $i$-th subregion around $\VEC{p}$ for $i=1,\cdots,\regionnumber$. Motivated by pixel diffusion-based methods fill the target region by propagating information along multiple directions, we propose to divide the neighborhood of $\VEC{p}$ into disjoint fan-shaped sectors (see sectors-partition in \TheFig{fig-rwm-ways}):
\begin{equation}\label{eq-dir-points}
  \Gamma_{i} = \left\{\VEC{p}+t\VEC{u}:0< t\leq r,\frac{2\pi (i-1)}{\regionnumber}\leq \theta_i < \frac{2\pi i}{\regionnumber}\right\},
\end{equation}
and $\VEC{u}=(\cos{\theta_i},\sin{\theta_i})^T$ .Then a group matrix $ \bm{G}_{\VEC{p}}$ is constructed by stacking all similar patches as columns ($n=\regionnumber+1$):
\begin{equation}\label{eq-Gp}
  \bm{G}_{\VEC{p}} = [\bm{I}_{\VEC{p}}(:), \bm{I}_{\VEC{q}_1}(:),\bm{I}_{\VEC{q}_2}(:),\cdots,\bm{I}_{\VEC{q}_\regionnumber}(:) ]\in \realR^{m\times n}.
\end{equation}
Besides the above sectors-partition, we consider dividing the neighbor area into equal-size grids and searching for one similar patch in each grid (grids-partition). \TheFig{fig-rwm-ways} illustrates the two partition approaches. We will give an ablation study of these partition strategies in later experiments.
\subsection{Iterative inpainting with region-wise matching}
\newcommand{\refset}{\mathcal{R}}
Given an incomplete image $\bm{I}$ with mask $\bm{M}$, we divide the image into overlapped target patches of size $\patchsize$ by sliding a window. Let $K$ be the number of target patches, and denote $\refset=\{\VEC{p}_i\}_{i=1}^{K}$ as the set of all target patches. For each $\bm{p}_i$, we use the above region-wise matching method to find similar patches of $\bm{I}_{\bm{p}_i}$. For convenience, let $\bm{Q}_{\VEC{p}_i}$ denote all positions of similar patches respect to $\bm{I}_{\bm{p}_i}$ :
\begin{equation}\label{eq-Q}
  \bm{Q}_{\VEC{p}_i}=\{\VEC{q}_{ij}:j=1,\cdots,\regionnumber\}\cup \{\VEC{p}_i\}.
\end{equation}
Then we construct a patch matrix $\MATRIX{Y}_{\VEC{p}_i}$ and its corresponding indicator mask $\bm{\Omega}_{\VEC{p}_i}$
\begin{equation}\label{eq-Gpi}
\begin{aligned}
  \MATRIX{Y}_{\VEC{p}_i} &= [\bm{I}_{\VEC{p}_i}(:), \bm{I}_{\VEC{q}_{i1}}(:),\bm{I}_{\VEC{q}_{i2}}(:),\cdots,\bm{I}_{\VEC{q}_{i\regionnumber}}(:) ],\\
  \MATRIX{\Omega}_{\VEC{p}_i} &= [\bm{M}_{\VEC{p}_i}(:), \bm{M}_{\VEC{q}_{i1}}(:),\bm{M}_{\VEC{q}_{i2}}(:),\cdots,\bm{M}_{\VEC{q}_{i\regionnumber}}(:) ].
\end{aligned}
\end{equation}
The proposed NC-WLRD model is then applied to complete $\MATRIX{Y}_{\VEC{p}_i}$:
 \begin{equation}\label{eq-model-pm}
 \begin{aligned}
   \MATRIX{X}_{\bm{p}_i},\MATRIX{B}_{\bm{p}_i} = &\mathop{\arg\min}_{\ML,\MS}\knn+\lambda\|\MATRIX{\Omega}_{\VEC{p}_i}\odot \MS\|_{1}~~\\
   &s.t.~~\MATRIX{Y}_{\VEC{p}_i}=\ML+\MS\\
   \end{aligned}
\end{equation}
\begin{figure}[htbp]
\centering
\newcommand{\fs}{0.26}
\newcommand{\figmargin}{-3mm}
\includegraphics[width=0.96\linewidth]{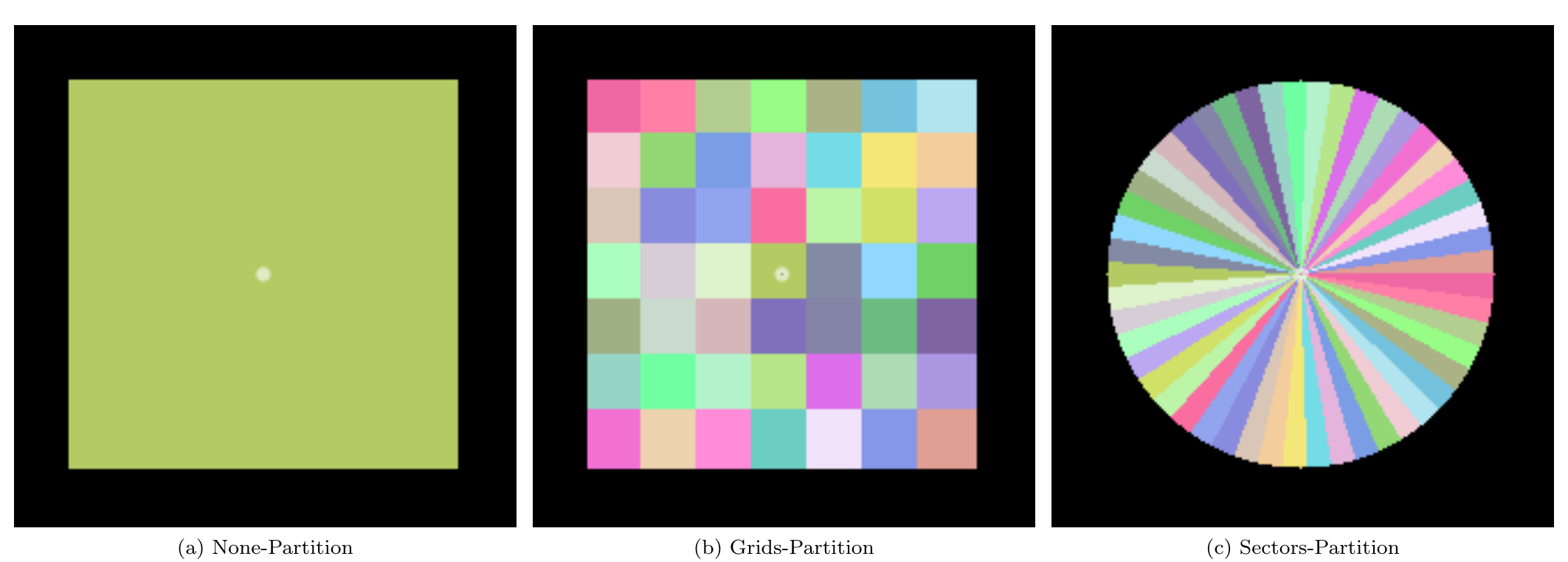}
\caption{Illustration of different ways for generating subregions. The white dot in the center represents the position of a target patch. (a) Traditional exhaustive way, finding $\regionnumber$ most similar patches within a neighbor window. (b) Region-wise matching with grids-partition. (c) Region-wise matching with sectors-partition. We search for one similar patch in each grid/sector.}\label{fig-rwm-ways}
\end{figure}
In our implementation, \eqref{eq-model-pm} is solved for $i=1,\cdots,K$ in parallel. The reconstructed patch matrices $\{\MATRIX{X}_{\bm{p}_i}\}_{i=1}^{K}$ are then aggregated by taking the average of all estimated values for each pixel, i.e the final reconstructed image $\hat{\bm{I}}$ is formulated as
\begin{equation}\label{eq-rec}
  \hat{\bm{I}} = \frac{\sum_{\bm{p}_i\in \refset}\sum_{\bm{q}\in \bm{Q}_{\VEC{p}_i}}\bm{R}_{\VEC{q}}^{T}({\MATRIX{X}}_{\bm{p}_i,\VEC{q}})} {\sum_{\bm{p}_i\in\refset}\sum_{\bm{q}\in \bm{Q}_{\VEC{p}_i}} \bm{R}_{\VEC{q}}^{T}(\bm{1})},
\end{equation}
where ${\MATRIX{X}}_{\bm{p}_i,\VEC{q}}$ is the reconstructed version of patch $\bm{I}_{\VEC{q}}$ in $\MATRIX{X}_{\bm{p}_i}$, and $\bm{1}$ is a patch of size $\patchsize$ with all its elements being 1. And the division is element-wise.

In our implementation, all possible locations in each subregion are pre-computed as relative displacements, namely, $t\VEC{u}$ in $\Gamma_{i}$ \eqref{eq-dir-points}.
A main advantage of our region-wise matching strategy is that it runs much faster than traditional exhaustive method since it does not require sorting patch similarity, and can be implemented in parallel for all subregions. Most significantly, the found similar patches contribute their observed elements to each other, and hence the proposed RwM-WLRD recoveries entire-row or entire-column missing images effectively.
The whole procedure of RwM-WLRD is summarized in algorithm \ref{alg-ms-lrd-inpainting}. For better performance, we iteratively implement the whole process until convergence. \TheFig{fig-flowchart} shows the algorithm flowchart.
\begin{algorithm}[ht]
\caption{Image Inpainting via RwM-WLRD }
\label{alg-ms-lrd-inpainting}

\hspace*{0.02in} {\bf Input:}
Incomplete Image $\bm{I}$, mask $\bm{M}$, $\lambda>0$

\hspace*{0.02in} {\bf Initialization:}
$\hat{\bm{I}}=\bm{I}$
\begin{algorithmic}[1]
\WHILE{not converged}
\STATE Divide image $\hat{\bm{I}}$ into $K$ target patches $\{\VEC{p}_i\}_{i=1}^{K}$

    \FOR{$i=1:K$}
        \STATE Find similar patches of $\bm{p}_i$ \eqref{eq-N-patches}
        \STATE Construct patch matrix $\MATRIX{Y}_{\bm{p}_i}$ and mask $\MATRIX{\Omega}_{\VEC{p}_i}$ \eqref{eq-Gp}
        \STATE Apply NC-WLRD model \eqref{eq-model} to solve 
        $$
        \begin{aligned}
   \MATRIX{X}_{\bm{p}_i},\MATRIX{B}_{\bm{p}_i} = &\mathop{\arg\min}_{\ML,\MS}\knn+\lambda\|\MATRIX{\Omega}_{\VEC{p}_i}\odot \MS\|_{1}~~\\
   &s.t.~~\MATRIX{Y}_{\VEC{p}_i}=\ML+\MS\\
   \end{aligned}
        $$
    \ENDFOR
    \STATE Aggregate and update $\hat{\bm{I}}$ using formula \eqref{eq-rec}
\ENDWHILE
\end{algorithmic}
\hspace*{0.02in} {\bf Output: }
Reconstructed Image $\hat{\bm{I}}$
\end{algorithm}
\begin{figure*}[!htb]
  \centering
  \includegraphics[width=0.96\linewidth]{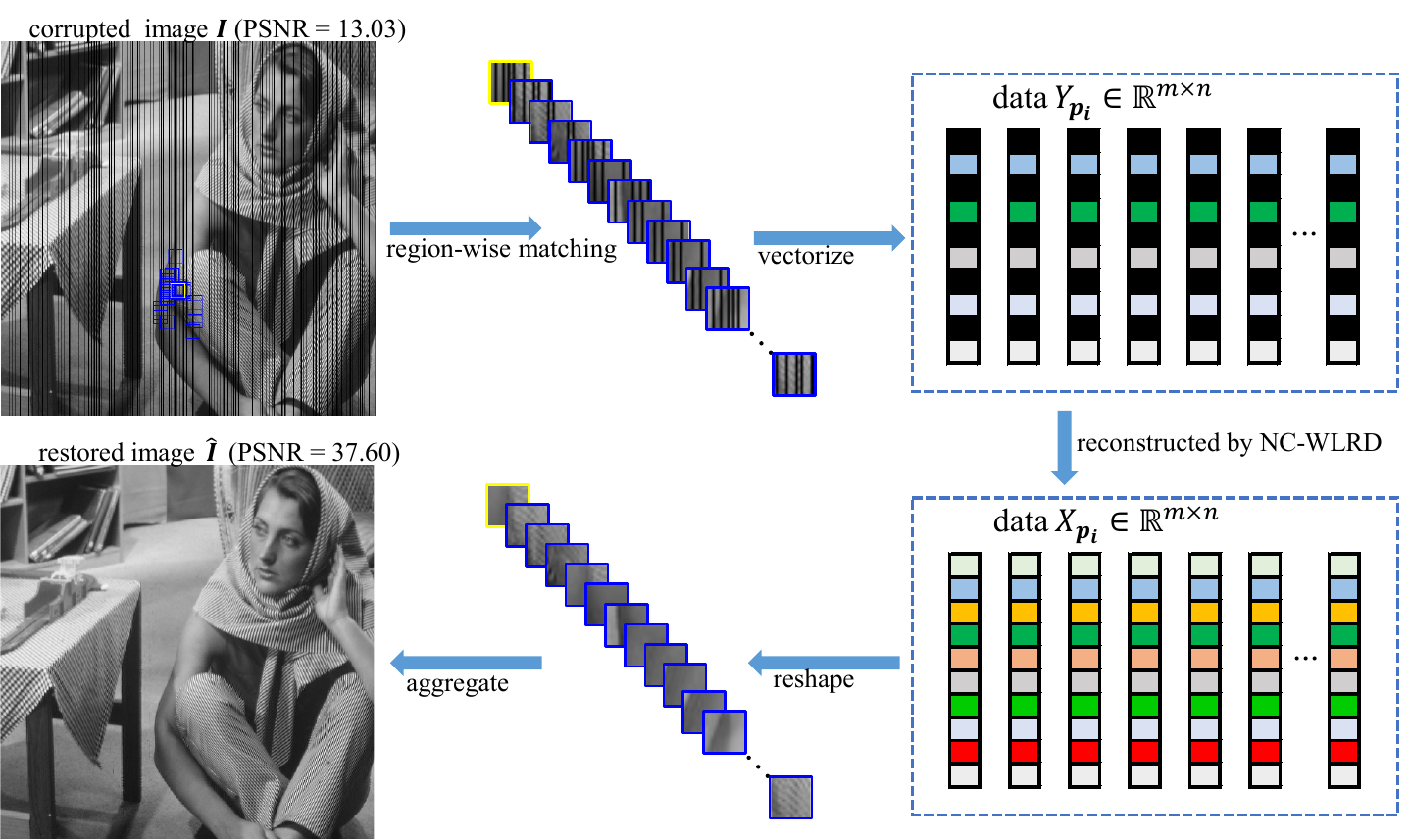}
  \caption{Flowchart of the proposed image inpainting algorithm RwM-WLRD. A group of similar patches is used as an example in the middle of the flowchart. The final aggregation is done after all groups are reconstructed.}
  \label{fig-flowchart}
\end{figure*}

\section{Experiments}
\label{section-exps}
In this section, we first evaluate the proposed NC-WLRD on image completion (not patch-based). Then we conduct various experiments on real images corrupted by horizontal and/or vertical lines to demonstrate the effectiveness of our region-wise matching strategy and show that the proposed RwM-WLRD has superior performance to other image inpainting algorithms.

For all experiments, $\lambda$ is set to 1 for noiseless inpainting. As for patch-based inpainting, the following patch configurations are set consistently: patch size $\sqrt{m}=8$, search radius $r=90$, number of similar patches (i.e. number of subregions, since we search only one similar patch in each subregion \eqref{eq-N-patches}) $\regionnumber=60$. We use MATLAB R2022b on a desktop computer with Intel(R) Core(TM) 3.6-GHz i7 CPU and 32-GB RAM to run the experiments. We select peak signal-to-noise ratio (PSNR) and structural similarity (SSIM) as quality metrics.
\begin{figure}[!htbp]
  \centering
  \includegraphics[width=0.96\linewidth]{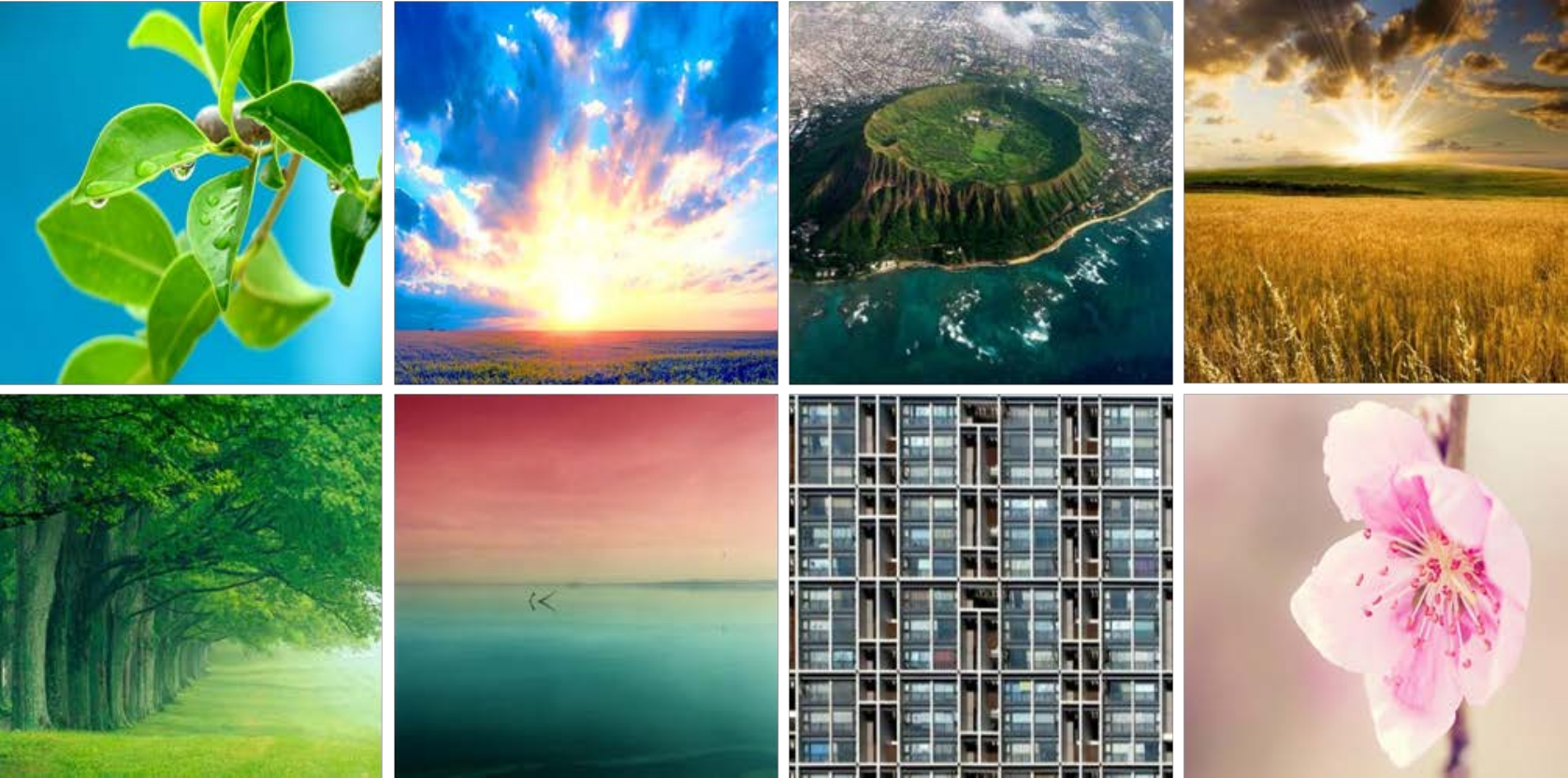}
  \caption{The original images for testing matrix completion performance.}\label{fig-usedpictures}
\end{figure}
\begin{figure*}[!htb]
  \centering
  \includegraphics[width=0.96\linewidth]{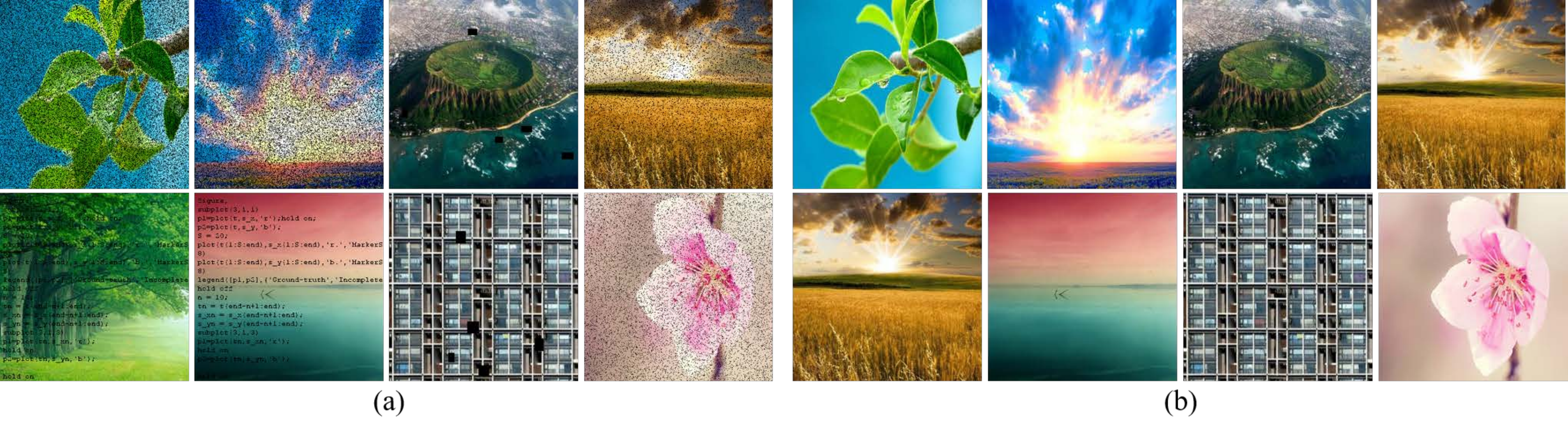}
  \caption{Image restoration results of NC-WLRD. (a) Different degradation masks overlaid on original images. (b) Completion results. }\label{fig-recpictures}
\end{figure*}
\begin{figure}[!htb]
  \centering
  \includegraphics[width=0.6\linewidth]{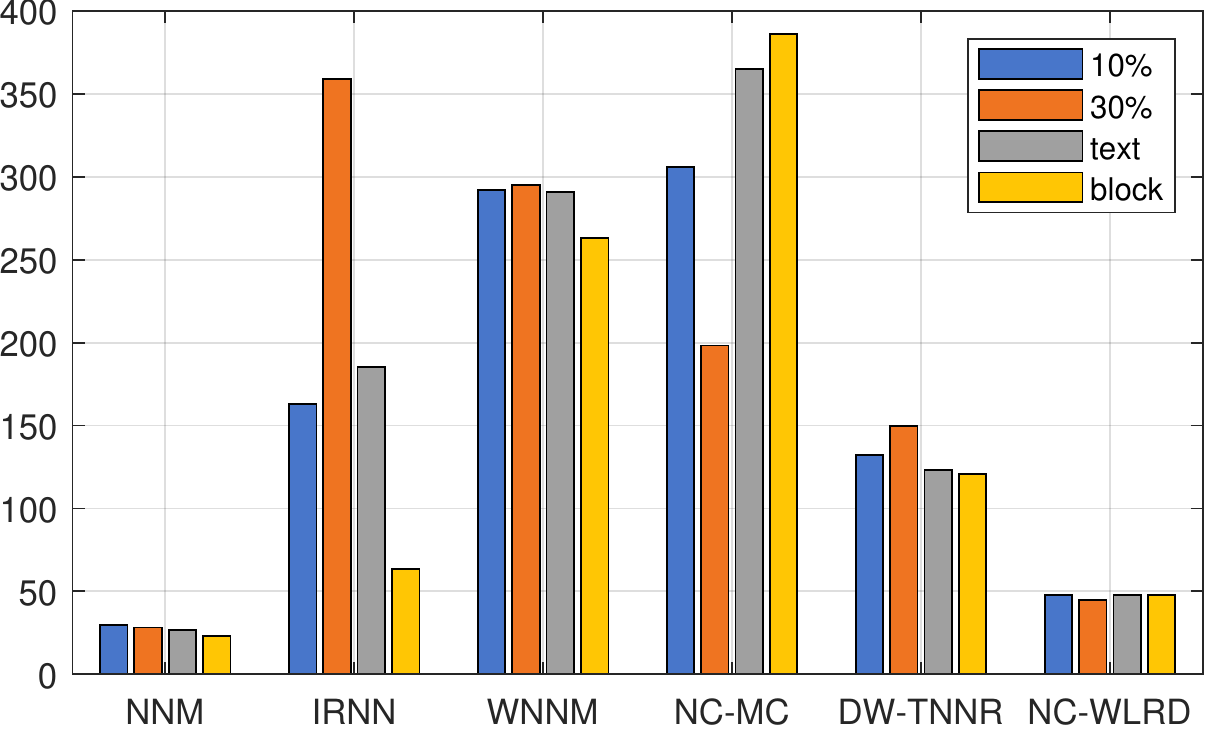}
  \caption{Comparison of average running time (in seconds) in different tasks. The proposed model has a better trade-off between reconstruction performance and speed.}\label{fig-runtime}
\end{figure}
\subsection{Restoration results on matrix completion}
To evaluate the performance of the proposed NC-WLRD on matrix completion, we consider four configurations of masks: random masks with 10\%, 30\% missing pixels, random block mask and a text mask. The random mask and random block mask are generated separately for each image, and the shape of each rectangle block is varied in the range $[10,20]$. The original test images are shown in \TheFig{fig-usedpictures}.
\interfootnotelinepenalty=100000
We compare the proposed NC-WLRD with several state-of-the-art methods,
including: NNM (based on IALM\footnote{https://people.eecs.berkeley.edu/\~{}yima/matrix-rank/sample\_code.html} \cite{lin2010augmented}), WNNM\footnote{http://www4.comp.polyu.edu.hk/\~{}cslzhang} \cite{gu2017weighted}, IRNN\footnote{https://github.com/canyilu/IRNN} \cite{Lu2016IRNN}, NC-MC\footnote{https://github.com/sudalvxin/MC-NC.git} \cite{Nie2019NCMC},
and DW-TNNR\footnote{https://github.com/xueshengke/DW-TNNR} \cite{xue2019double}. We use the default settings for other methods, as suggested in their original papers. $L_p$ penalty is used with $p=0.5$ in IRNN \cite{Lu2016IRNN}.

1) Completion results: \TheTable{tab:imcresults}
presents the PSNR and SSIM results respectively. As demonstrated, the proposed NC-WLRD outperforms other methods in most cases and achieves the highest average scores in both PSNR and SSIM metrics. Though in several cases the proposed NC-WLRD is slightly lower than DW-TNNR \cite{xue2019double}, it outperforms DW-TNNR remarkably by more than 2 dB on average. \TheFig{fig-recpictures} further visualizes some restoration results of the proposed NC-WLRD.

2) Computational complexity analysis: As described in algorithm \ref{alg-admm-rpca}, the proposed NC-WLRD iteratively updates the low-rank and sparse components alternatively.
The computation of $\MS^{(k+1)}$ and $A^{(k+1)}$ require only matrix addition and element-wise multiplication and hence the complexity is $\mathcal{O}({mn})$. The main computational burden within the loop environment is the SVD calculation of $D^{(k)}$ to update $\ML^{(k+1)}$, which has complexity $\mathcal{O}{(mn^2)}$. Therefore, the overall complexity is $\mathcal{O}{(mn^{2})}$.
In the above experiments, the residual error $\| \MY - \ML^{(k+1)} -\MS^{(k+1)} \|_{F}/\| \MY \|_{F}<10^{-7}$ is adopted as the convergence criterion. The proposed method takes about 188.35 seconds to complete all images (there are five images of shape $300\times 300$, and the other three are of shape $300\times400$) with all masks, the speed is much faster than WNNM \cite{gu2017weighted} and NC-MC \cite{Nie2019NCMC}. \TheFig{fig-runtime} exhibits the total running time on each type of mask. As demonstrated in \TheFig{fig-runtime}, except NNM \cite{lin2010augmented}, the proposed NC-WLRD takes much less time to recover incomplete images on each type of mask than other algorithms, at the same time it also achieves high restoration quality.
\begin{table*}[htbp]
  \centering
  \caption{The PSNR and SSIM results on image completion with four types of masks: random mask with 10\%, 30\% missing pixels, random block mask and text mask. The best PSNR and SSIM are highlighted in bold and underlined respectively. The proposed method obtains the best PSNR score (40.97) on average of 32 experiments.}
   \begin{tabular}{cccccccccccccc}
    \toprule
    Image & Mask  & \multicolumn{2}{c}{NNM} & \multicolumn{2}{c}{WNNM} & \multicolumn{2}{c}{IRNN} & \multicolumn{2}{c}{DW-TNNR} & \multicolumn{2}{c}{MCNC} & \multicolumn{2}{c}{NC-WLRD} \\
    \midrule
    \multirow{4}[1]{*}{1} & 10\%   & 37.94  & 0.9773  & 35.78  & 0.9561  & 37.77  & 0.9735  & 38.97  & 0.9833  & 36.04  & 0.9443  & \textbf{39.06 } & \underline{0.9838 } \\
          & 30\%   & 20.89  & 0.4757  & 29.66  & 0.8561  & 31.08  & 0.8890  & 32.35  & 0.9228  & 31.46  & 0.8902  & \textbf{32.53 } & \underline{0.9262 } \\
          & text  & 32.49  & 0.9377  & 39.30  & 0.9843  & 30.68  & 0.9363  & 40.78  & 0.9907  & 37.45  & 0.9548  & \textbf{40.97 } & \underline{0.9918 } \\
          & block & 28.85  & 0.9864  & 36.99  & \underline{0.9926 } & 24.99  & 0.9830  & 35.88  & 0.9915  & 35.56  & 0.9563  & \textbf{37.38 } & 0.9921  \\
    \multirow{4}[0]{*}{2} & 10\%   & 39.29  & 0.9804  & 37.35  & 0.9731  & 39.53  & 0.9827  & 40.31  & 0.9889  & 35.60  & 0.9470  & \textbf{40.31 } & \underline{0.9889 } \\
          & 30\%   & 16.14  & 0.2954  & 31.49  & 0.9021  & 33.27  & 0.9189  & \textbf{34.52 } & \underline{0.9518 } & 32.33  & 0.9041  & 34.51  & 0.9514  \\
          & text  & 28.29  & 0.9062  & 42.28  & 0.9866  & 25.24  & 0.8945  & 44.69  & 0.9927  & 37.10  & 0.9559  & \textbf{44.93 } & \underline{0.9934 } \\
          & block & 27.88  & 0.9850  & 41.39  & 0.9935  & 22.83  & 0.9796  & 38.42  & 0.9927  & 37.31  & 0.9609  & \textbf{44.50 } & \underline{0.9963 } \\
    \multirow{4}[0]{*}{3} & 10\%   & 34.52  & 0.9636  & 31.19  & 0.9067  & 33.79  & 0.9457  & 34.82  & \underline{0.9648 } & 32.75  & 0.9103  & \textbf{34.83 } & 0.9644  \\
          & 30\%   & 24.87  & 0.7727  & 24.77  & 0.6896  & 26.64  & 0.7670  & \textbf{28.47 } & 0.8558  & 27.68  & 0.7960  & 28.44  & \underline{0.8808 } \\
          & text  & 34.48  & 0.9693  & 33.24  & 0.9511  & 33.16  & 0.9598  & 36.92  & 0.9767  & 33.90  & 0.9218  & \textbf{36.98 } & \underline{0.9771 } \\
          & block & 39.61  & 0.9922  & 34.13  & 0.9883  & 32.43  & 0.9886  & 41.15  & 0.9923  & 34.87  & 0.9399  & \textbf{41.26 } & \underline{0.9927 } \\
    \multirow{4}[0]{*}{4} & 10\%   & 37.44  & 0.9839  & 34.72  & 0.9714  & 37.27  & 0.9832  & 37.73  & 0.9847  & 37.01  & 0.9792  & \textbf{37.77 } & \underline{0.9849 } \\
          & 30\%   & 23.57  & 0.6605  & 28.22  & 0.8552  & 30.08  & 0.8925  & 31.20  & 0.9216  & 31.21  & 0.9200  & \textbf{31.25 } & \underline{0.9219 } \\
          & text  & 30.33  & 0.9294  & 38.43  & 0.9830  & 28.20  & 0.9279  & 40.37  & 0.9891  & 39.81  & 0.9865  & \textbf{40.51 } & \underline{0.9894 } \\
          & block & 33.65  & 0.9861  & 35.69  & 0.9885  & 27.95  & 0.9796  & 39.58  & 0.9922  & 39.29  & \underline{0.9923 } & \textbf{39.58 } & 0.9922  \\
    \multirow{4}[0]{*}{5} & 10\%  & 38.21  & 0.9780  & 35.11  & 0.9546  & 37.87  & 0.9744  & 38.32  & 0.9785  & 31.96  & 0.9117  & \textbf{38.35 } & \underline{0.9786 } \\
          & 30\%   & 27.29  & 0.7599  & 28.07  & 0.7913  & 30.06  & 0.8480  & \textbf{31.56 } & \underline{0.8977 } & 29.26  & 0.8367  & 31.53  & 0.8958  \\
          & text  & 38.71  & 0.9808  & 36.90  & 0.9723  & 37.58  & 0.9760  & 39.83  & 0.9849  & 32.53  & 0.9221  & \textbf{39.92 } & \underline{0.9851 } \\
          & block & 38.49  & 0.9912  & 36.68  & 0.9889  & 32.60  & 0.9865  & \textbf{40.76 } & 0.9922  & 32.89  & 0.9325  & 40.74  & \underline{0.9922 } \\
    \multirow{4}[0]{*}{6} & 10\%   & \textbf{56.18 } & 0.9990  & 53.52  & 0.9975  & 51.20  & 0.9970  & 56.16  & 0.9991  & 49.46  & 0.9944  & 56.17  & \underline{0.9991 } \\
          & 30\%   & 30.21  & 0.7391  & 46.44  & 0.9927  & 46.79  & 0.9920  & 48.78  & 0.9961  & 44.77  & 0.9857  & \textbf{48.82 } & \underline{0.9961 } \\
          & text  & 32.32  & 0.9226  & 58.50  & 0.9986  & 27.41  & 0.8928  & 44.53  & 0.9816  & 54.35  & 0.9972  & \textbf{64.38 } & \underline{0.9996 } \\
          & block & 33.57  & 0.9894  & 57.87  & 0.9989  & 23.45  & 0.9738  & 36.01  & 0.9885  & 56.81  & 0.9987  & \textbf{63.44 } & \underline{0.9998 } \\
    \multirow{4}[0]{*}{7} & 10\%   & 39.00  & 0.9974  & 37.30  & 0.9953  & 38.70  & 0.9972  & 39.89  & 0.9977  & 36.94  & 0.9951  & \textbf{39.98 } & \underline{0.9977 } \\
          & 30\%   & 22.04  & 0.9069  & 30.50  & 0.9773  & 31.35  & 0.9827  & 32.78  & 0.9870  & 31.90  & 0.9840  & \textbf{32.93 } & \underline{0.9873 } \\
          & text  & 31.77  & 0.9866  & 39.55  & 0.9964  & 30.91  & 0.9866  & 41.78  & 0.9978  & 39.16  & 0.9962  & \textbf{42.07 } & \underline{0.9979 } \\
          & block & 31.50  & 0.9942  & 37.92  & 0.9972  & 24.33  & 0.9802  & 41.36  & 0.9983  & 39.90  & 0.9978  & \textbf{41.55 } & \underline{0.9984 } \\
    \multirow{4}[0]{*}{8} & 10\%   & 40.79  & 0.9813  & 39.24  & 0.9725  & 40.69  & 0.9807  & 42.48  & 0.9910  & 41.30  & 0.9872  & \textbf{42.53 } & \underline{0.9911 } \\
          & 30\%   & 15.73  & 0.1204  & 33.75  & 0.9219  & 34.27  & 0.9263  & 36.66  & 0.9650  & 35.93  & 0.9585  & \textbf{36.73 } & 0.9652  \\
          & text  & 28.36  & 0.8854  & 40.95  & 0.9872  & 25.39  & 0.8719  & 41.12  & 0.9878  & 44.28  & 0.9948  & \textbf{44.58 } & \underline{0.9953 } \\
          & block & 28.97  & 0.9842  & 40.85  & 0.9943  & 22.93  & 0.9798  & 38.00  & 0.9929  & 43.84  & \underline{0.9963 } & \textbf{42.58 } & 0.9958  \\
    Average &       & 31.98  & 0.8756  & 37.74  & 0.9536  & 32.20  & 0.9484  & 38.94  & 0.9759  & 37.64  & 0.9515  & \textbf{40.97 } & \underline{0.9782 } \\
    \bottomrule
    \end{tabular}%
  \label{tab:imcresults}%
\end{table*}%

We also evaluate the proposed NC-WLRD in a larger image dataset. We use BSD68 \cite{roth2009fields} for performance evaluation, BSD68 is widely used for denoising \cite{zhang2017beyond,zhang2018ffdnet} and it contains 68 natural images from Berkeley segmentation dataset. For each image in the dataset, parts of all pixels are set unobserved randomly with different missing rates (0.1,0.2,0.4,0.6), and the average PSNR values on all images are reported for each missing rate. Comparison results are shown in \TheTable{tab:bsd68mc}. The proposed NC-WLRD achieves the highest PSNR value on average.
\begin{table}[htbp]
  \centering
  \caption{Completion results (PSNR) in image dataset BSD68 for different missing rates.}
    \begin{tabular}{cccccc}
    \toprule
    \multirow{2}[4]{*}{Method} & \multicolumn{4}{c}{Missing rate} &  \multirow{2}[4]{*}{Average}\\
\cmidrule{2-5}          & 0.1   & 0.2   & 0.4   & 0.6   &  \\
    \midrule
    NNM   & 35.54  & 28.01  & 18.25  & 13.16  & 23.74  \\
    WNNM  & 33.51  & 29.71  & 25.22  & 23.27  & 27.93  \\
    IRNN  & 35.50  & 31.34  & 25.80  & 20.24  & 28.22  \\
    DW-TNNR & 36.77  & 32.85  & 28.23  & 25.37  & 30.81  \\
    NC-MC  & 36.57  & 32.92  & \textbf{28.49}  & 25.40  & 30.84  \\
    NC-WLRD & \textbf{36.82}  & \textbf{32.93}  & 28.28  & \textbf{25.61}  & \textbf{30.91}  \\
    \bottomrule
    \end{tabular}%
  \label{tab:bsd68mc}%
\end{table}%

\subsection{Image inpainting with RwM-WLRD}\label{LineInpainting}
\textbf{Line inpainting of grayscale image}:
Restoration of entire-row/column missing (also known as horizontal or vertical dead lines) is a challenge for most low-rank based inpainting methods. To demonstrate the effectiveness of the proposed region-wise matching strategy, we evaluate our model on dead line inpainting.
We compare the proposed RwM-WLRD on line inpainting with following competing methods: BPFA\footnote{http://people.ee.duke.edu/\~{}lcarin/BCS.html} \cite{Zhou2012BPFA}, WNNM \cite{gu2017weighted}, TSLRA \cite{guo2017patch}, NLM \cite{Newson2017}, TGV \cite{wali2019new} and LRTV-MADC\footnote{https://github.com/ZhengJianwei2/MADC} \cite{Zheng2020MADC}. The test image set is shown in \TheFig{set1}, and we randomly generate a missing line mask separately for each image in the dataset. Comparison results on PSNR and SSIM metrics are presented in \TheTable{tab-line-bm3d-psnr} and \TheTable{tab-line-bm3d-ssim} respectively.
\begin{figure}[!htbp]
  \centering
  \includegraphics[width=0.96\linewidth]{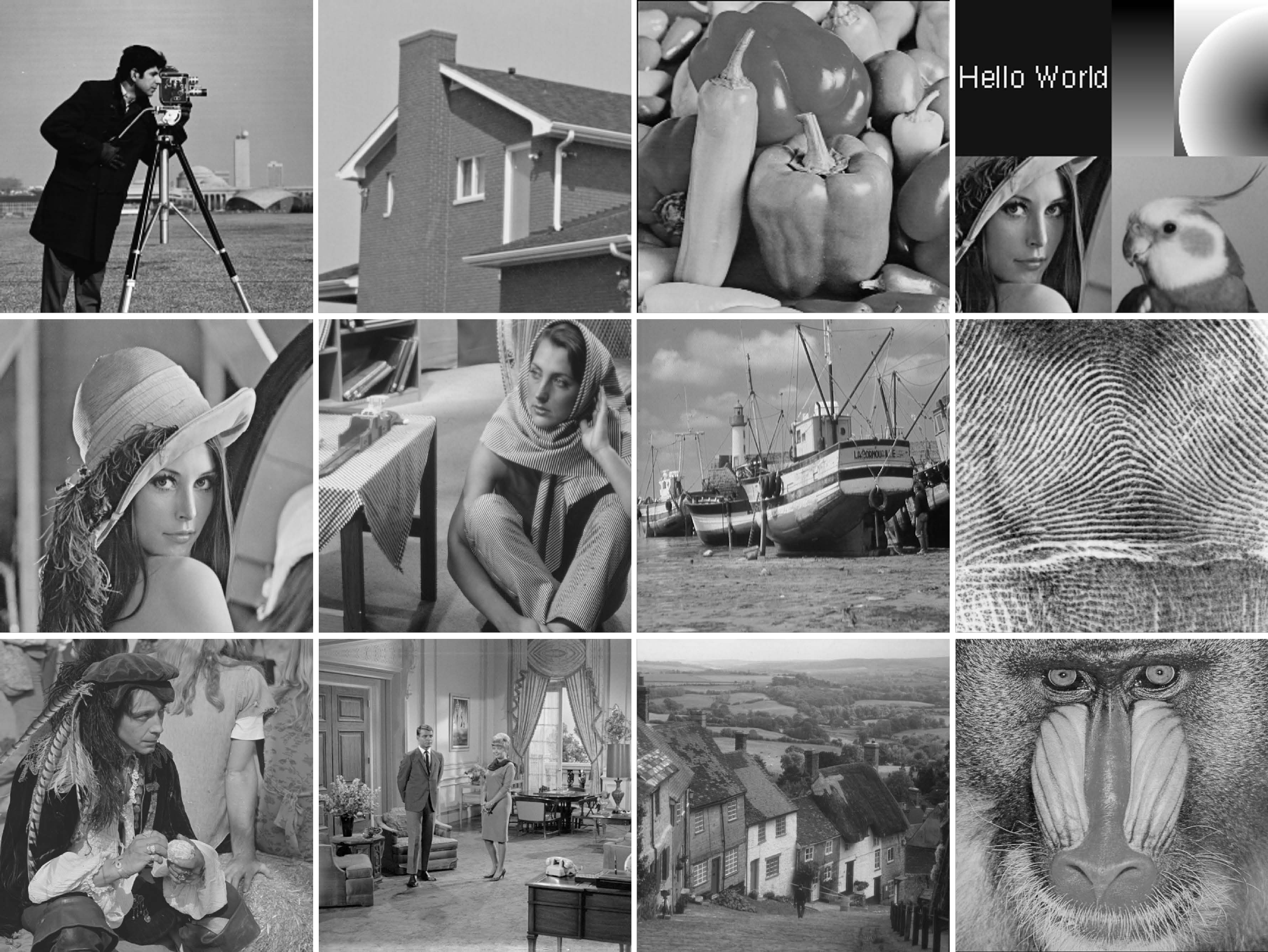}
  \caption{Grayscale image dataset BM3D \cite{BM3D} for evaluating performance on line inpainting. Size: $256\times256$ for images in the first row and others are of size $512\times512$.}\label{set1}
\end{figure}
\begin{figure*}[!htbp]
\centering
\newcommand{\fs}{0.22}
\newcommand{\figmargin}{-3mm}
\includegraphics[width=0.96\linewidth]{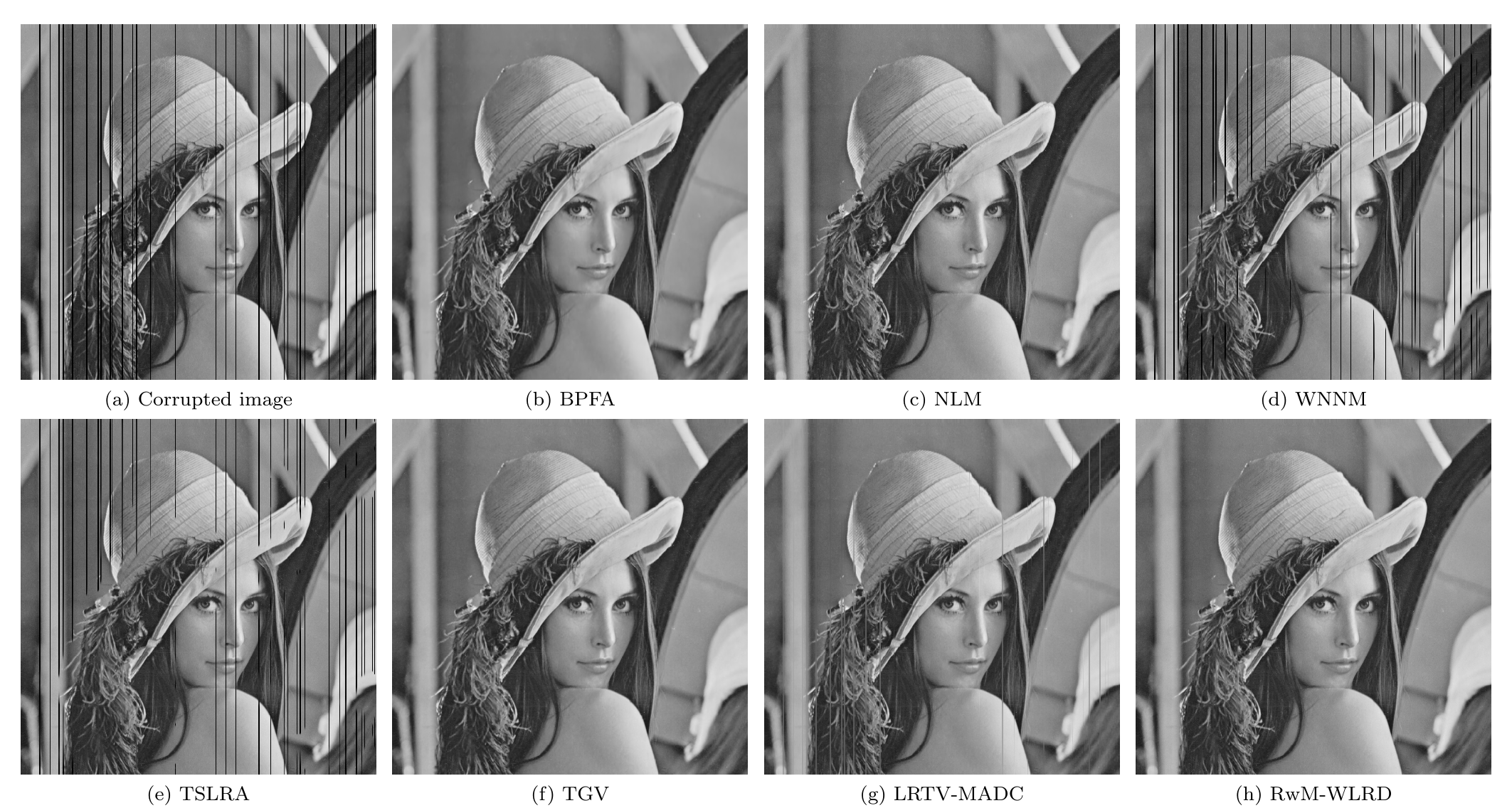}
\caption{Line inpainting results of \textit{Lena} image ($512\times512$).}\label{fig-lineresult-ex2}
\end{figure*}

\begin{figure*}[!htbp]
\centering
\newcommand{\fs}{0.22}
\newcommand{\figmargin}{-3mm}
\includegraphics[width=0.96\linewidth]{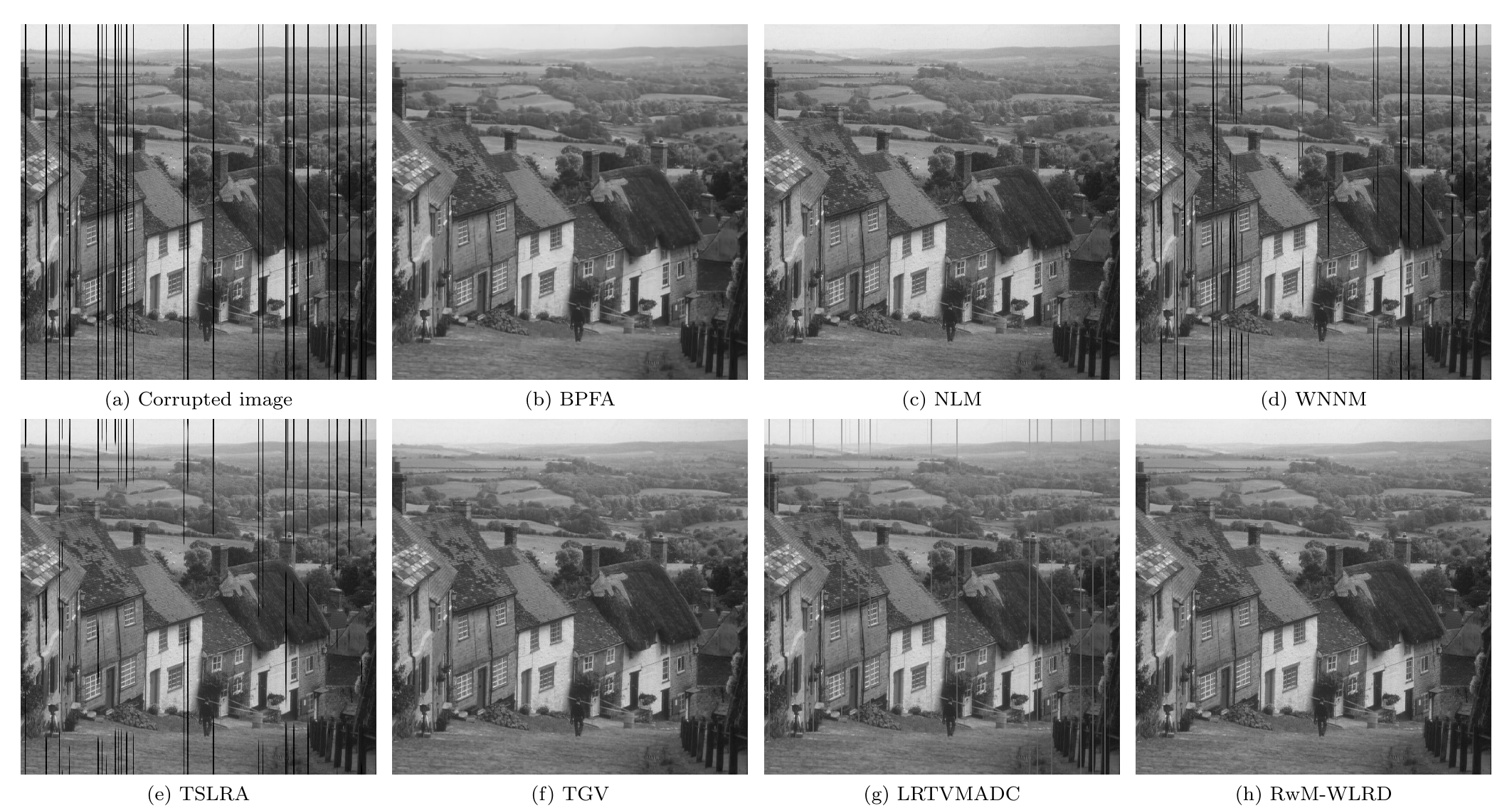}
\caption{Line inpainting results of \textit{Hill} image ($512\times512$).}\label{fig-lineresult-ex3}
\end{figure*}

\begin{table*}[h]
  \centering
  \caption{Line inpainting results of PSNR (dB) in image dataset BM3D.}
    \begin{tabular}{cccccccc}
    \toprule
          & TGV   & WNNM  & TSLRA & NLM   & LRTV-MADC & BPFA  & RwM-WLRD \\
    \midrule
    \textit{C.man} & 31.27  & 14.86  & 16.40  & 31.37  & 26.83  & 31.32  & \textbf{32.60 } \\
    \textit{House} & 44.81  & 14.17  & 15.29  & 42.42  & 23.43  & 43.26  & \textbf{45.49 } \\
    \textit{Peppers} & 30.43  & 16.44  & 18.22  & 31.64  & 29.44  & 32.51  & \textbf{33.00 } \\
    \textit{Montage} & 35.04  & 15.61  & 16.52  & 33.71  & 27.02  & 35.08  & \textbf{35.50 } \\
    \textit{Lena} & 40.33  & 17.22  & 17.52  & 38.50  & 32.58  & 38.52  & \textbf{40.55 } \\
    \textit{Barbara} & 33.27  & 32.51  & 15.60  & 18.06  & 29.41  & 29.79  & \textbf{38.22 } \\
    \textit{Boats} & 36.66  & 17.94  & 18.65  & 35.95  & 31.72  & 35.56  & \textbf{37.12 } \\
    \textit{F.print} & 36.32  & 20.41  & 24.44  & 33.56  & 27.81  & 37.38  & \textbf{38.34 } \\
    \textit{Man} & 38.23  & 19.22  & 21.09  & 36.64  & 31.34  & 37.29  & \textbf{38.30 } \\
    \textit{Couple} & 36.41  & 17.40  & 18.65  & 34.87  & 26.25  & 35.29  & \textbf{36.47 } \\
    \textit{Hill} & 39.34  & 20.19  & 20.71  & 37.92  & 31.43  & 37.19  & \textbf{40.01 } \\
    \textit{Baboon} & 33.40  & 18.38  & 19.86  & 31.64  & 29.51  & 31.90  & \textbf{34.30 } \\
    Average & 36.29  & 18.69  & 18.58  & 33.86  & 28.90  & 35.42  & \textbf{37.49 } \\
    \bottomrule
    \end{tabular}
  \label{tab-line-bm3d-psnr}%
\end{table*}
\begin{table*}[h]
  \centering
  \caption{Line inpainting results of SSIM in image dataset BM3D.}
  \begin{tabular}{cccccccc}
    \toprule
          & TGV   & WNNM  & TSLRA & LRTV-MADC & NLM   & BPFA  & RwM-WLRD  \\
    \midrule
    \textit{C.man} & 0.9664  & 0.5507  & 0.6930  & 0.8675  & 0.9626  & 0.9530  & \textbf{0.9712 } \\
    \textit{House} & 0.9890  & 0.4328  & 0.5851  & 0.6862  & 0.9885  & 0.9799  & \textbf{0.9936 } \\
    \textit{Peppers} & 0.9739  & 0.6326  & 0.7547  & 0.9337  & 0.9613  & 0.9553  & \textbf{0.9758 } \\
    \textit{Montage} & 0.9903  & 0.6368  & 0.6952  & 0.8983  & 0.9874  & 0.9846  & \textbf{0.9913 } \\
    \textit{Lena} & 0.9841  & 0.6621  & 0.6913  & 0.9363  & 0.9812  & 0.9551  & \textbf{0.9888 } \\
    \textit{Barbara} & 0.9596  & 0.9584  & 0.5499  & 0.7099  & 0.9460  & 0.9409  & \textbf{0.9860 } \\
    \textit{Boats} & 0.9741  & 0.7315  & 0.7530  & 0.9328  & 0.9745  & 0.9484  & \textbf{0.9828 } \\
    \textit{F.print} & 0.9934  & 0.9018  & 0.9715  & 0.9506  & 0.9815  & 0.9913  & \textbf{0.9938 } \\
    \textit{Man} & 0.9822  & 0.7372  & 0.8138  & 0.9197  & 0.9765  & 0.9672  & \textbf{0.9848 } \\
    \textit{Couple} & 0.9751  & 0.7082  & 0.7755  & 0.8607  & 0.9669  & 0.9506  & \textbf{0.9777 } \\
    \textit{Hill} & 0.9817  & 0.8084  & 0.8450  & 0.9209  & 0.9759  & 0.9537  & \textbf{0.9859 } \\
    \textit{Baboon} & 0.9761  & 0.8018  & 0.8440  & 0.9158  & 0.9635  & 0.9267  & \textbf{0.9797 } \\
    Average & 0.9788  & 0.7135  & 0.7477  & 0.8777  & 0.9722  & 0.9589  & \textbf{0.9843 } \\
    \bottomrule
    \end{tabular}%
  \label{tab-line-bm3d-ssim}
\end{table*}%
\begin{figure*}[!htbp]
\centering
\newcommand{\fs}{0.22}
\newcommand{\pw}{0.2}
\newcommand{\figmargin}{-3mm}
\includegraphics[width=0.96\linewidth]{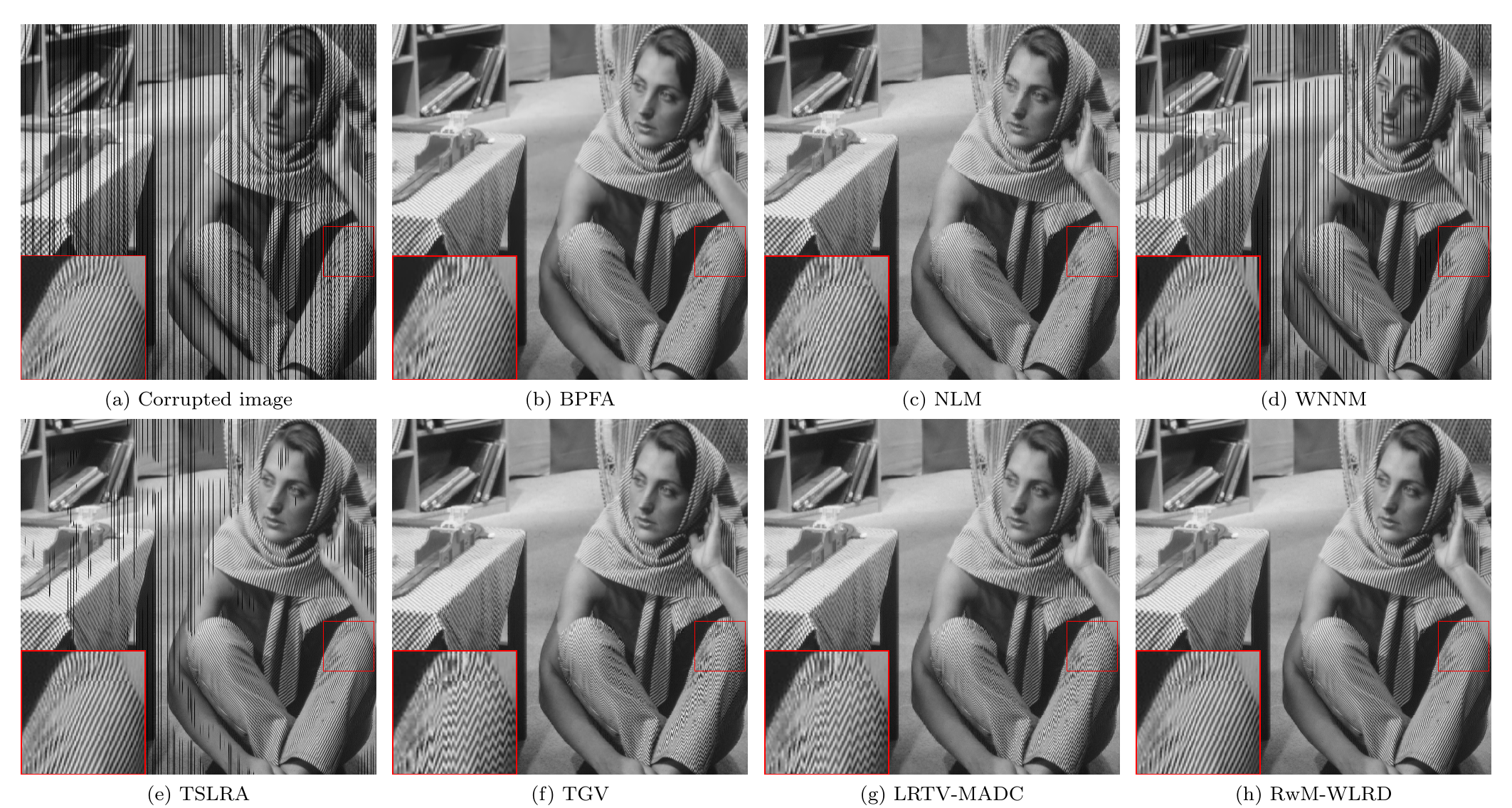}
\caption{Line inpainting results of \textit{Barbara} image ($512\times512$). In the corrupted image (a), the demarcated area is enlarged using the corresponding original data for better visualization.}\label{fig-lineresult-ex-zoom}
\end{figure*}
\begin{figure*}[!htbp]
\centering
\includegraphics[width=0.96\linewidth]{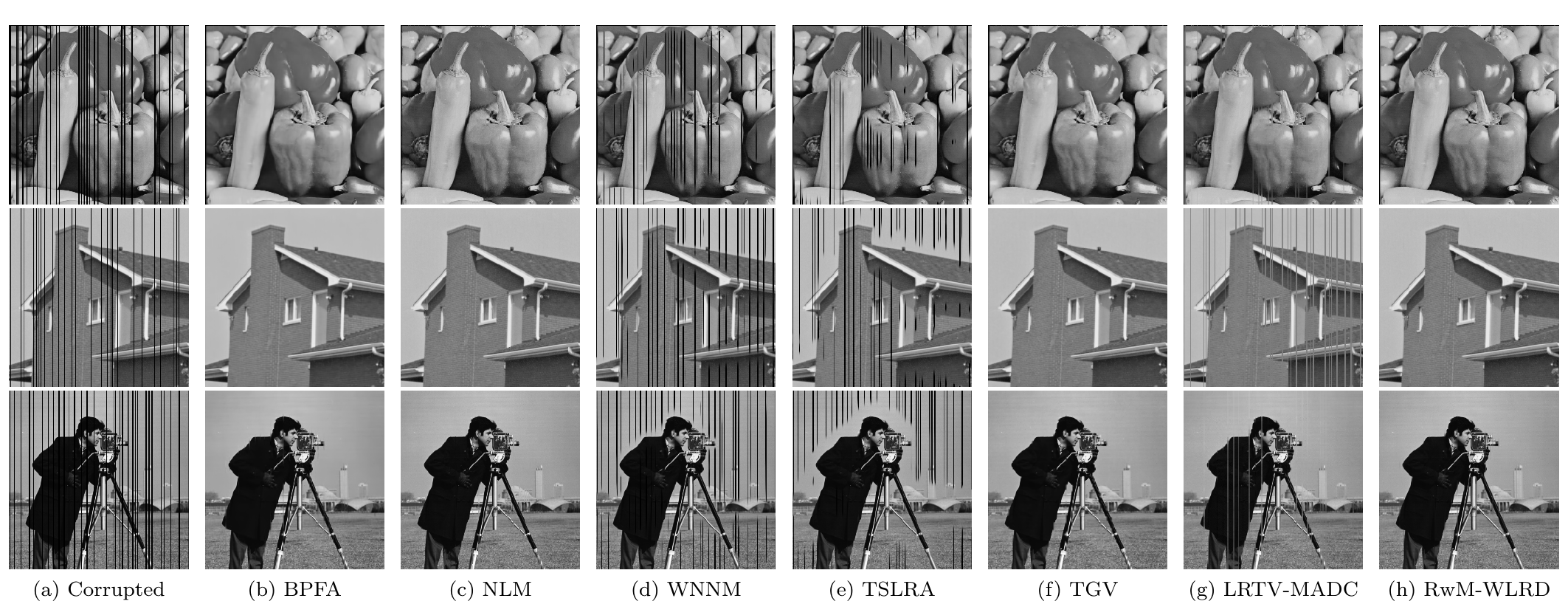}
\caption{Line inpainting results of \textit{Peppers, House} and \textit{Cameraman} images ($256\times 256$).}\label{fig-lineresult-peppers}
\end{figure*}
As shown in \TheTable{tab-line-bm3d-psnr} and \TheTable{tab-line-bm3d-ssim}, the proposed RwM-WLRD method has the best performance in terms of both PSNR and SSIM metrics, and it obtains the highest PSNR value 37.49 and SSIM score 0.9843 on average. Unsurprisingly, both WNNM \cite{gu2017weighted} and TSLRA \cite{guo2017patch} perform badly and fail to reconstruct all missing columns well.

\TheFig{fig-lineresult-ex2} - \TheFig{fig-lineresult-ex3} present the visual comparison results. It is clear that the proposed RwM-WLRD reconstructs missing lines perfectly in texture and smooth regions. However, the inpainted results of WNNM and TSLRA are not satisfactory, retaining lots of segments un-inpainted. To better understand the superiority of the proposed RwM-WLRD, we crop and enlarge a part of the \textit{Barbara} image, which is full of regular textures, and results are shown in \TheFig{fig-lineresult-ex-zoom}. Obviously, BPFA and LRTV-MADC produce blurred results and vertical line artifacts. The diffusion based method TGV, performs badly in this case and produces blocking artifacts. Although TSLRA is competitive with the proposed RwM-WLRD in the rectangle part, it can't fill missing lines in other regions.

\begin{figure*}[!htbp]
\centering
\newcommand{\fs}{0.32}
\newcommand{\figmargin}{-3mm}
\includegraphics[width=0.96\linewidth]{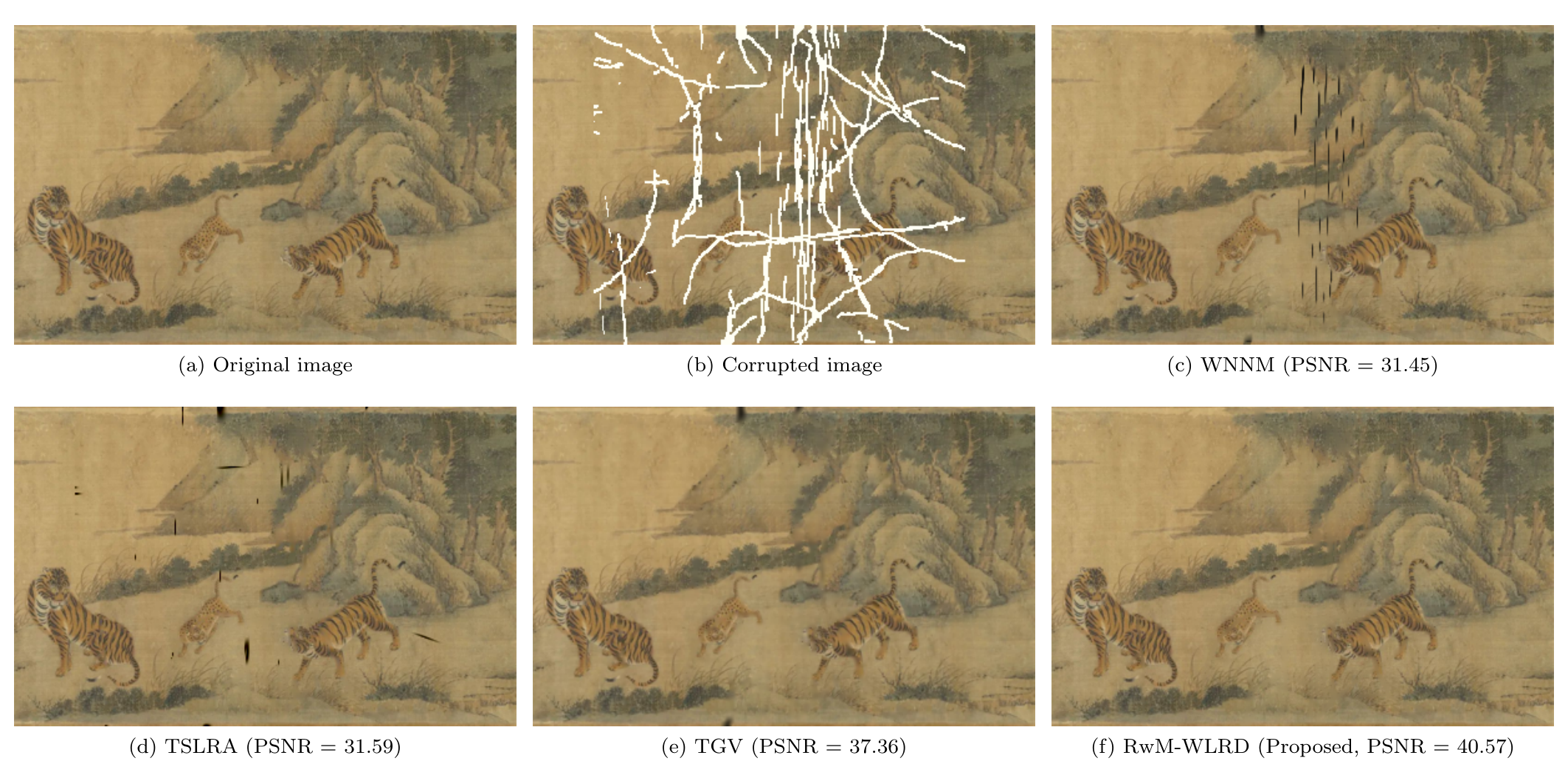}
\caption{Inpainting of scratches.}\label{fig-inpainting-result-hybrid}
\end{figure*}
\textbf{Inpainting of hybrid missing regions}: Although the proposed region-wise matching algorithm is specifically designed to solve the problem of line inpainting, it also works well for other types of missing masks, including text, scratch and blotch etc. \TheFig{fig-inpainting-result-hybrid} presents a comparison for repairing scratches. It is very clearly that the proposed method is superior to other competing methods and restores well all contaminated pixels.
The proposed RwM-WLRD is also effective for removing overlaid text, as illustrated in \TheFig{fig-text-our}. In summary, the proposed region-wise matching algorithm not only solves the problem of line inpainting (entire-missing columns) but also achieves robust and excellent performance in recovering other types of corrupted regions(scratches, text, etc.).

\begin{figure*}[!htbp]
\centering
\newcommand{\fs}{0.32}
\newcommand{\figmargin}{-3mm}
\includegraphics[width=0.96\linewidth]{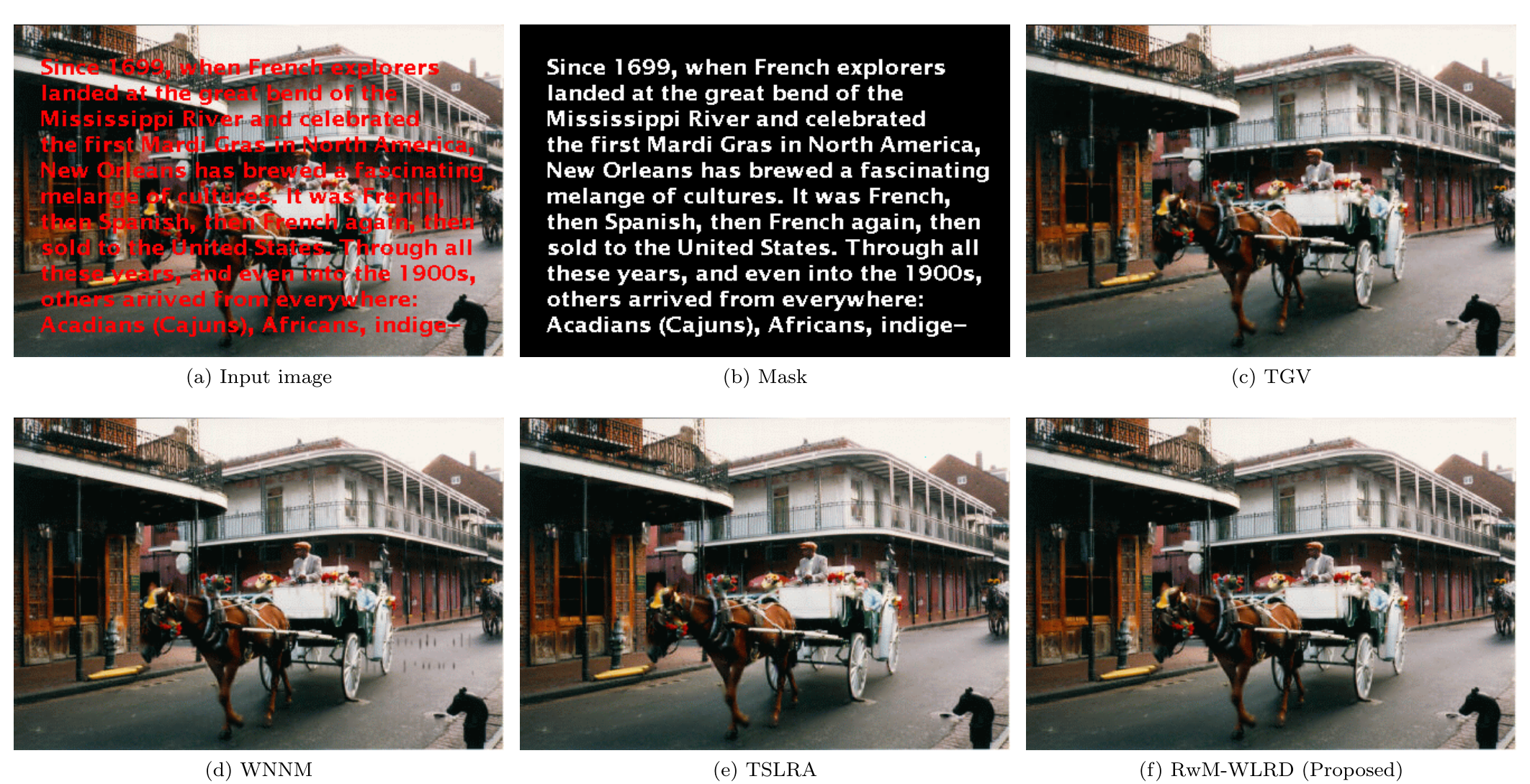}
\caption{Results of text removal using the proposed RwM-WLRD and other methods.}\label{fig-text-our}
\end{figure*}

\subsection{Comparison with deep learning method}
In this part, we compare the proposed algorithm with three deep learning-based methods for line inpainting, i.e. deep image prior (DIP) \cite{ulyanov2018deep}, free-form image inpainting (Deepfill2) \cite{yu2019free}, and conditional texture and structure dual generation (CTSDG) \cite{Guo_2021_ICCV}. We conduct experiments in BM3D and BSD68 datasets to evaluate performance more precisely. \TheTable{tab-comparison-deeplearning} presents the results of mean PSNR and SSIM scores on the two datasets. To demonstrate that the proposed algorithm is significantly superior to the three deep learning methods, we use the Kruskal-Wallis test and calculate the $p$-value (represents the probability of data samples coming from the same distribution) for PSNR scores. As shown in \TheTable{tab-comparison-deeplearning}, the proposed RwM-WLRD exceeds deep learning methods a large gap in terms of PSNR and SSIM, and the corresponding $p-$values are all less than 0.05, hence the proposed method is significant superior than trained deep models.

We also present several inpainted examples of dataset BSD68 in \TheFig{fig-lineresult-dl-bsd68} for visual comparison. It is obvious that CTSDG can't restore missing lines in smooth regions and produces blurry lines. Deepfill2 performs better than CTSDG in homogeneous regions but it outputs artifacts in detailed area. On the contrary, the proposed RwM-WLRD restores all corrupted pixels perfectly.
\begin{table}[!htbp]
  \centering
  \setlength{\tabcolsep}{1mm}
  \caption{Results of comparison with deep learning-based methods.}
    \begin{tabular}{ccccccc}
    \toprule
          & \multicolumn{3}{c}{BSD68} & \multicolumn{3}{c}{BM3D} \\
    \midrule
    method & PSNR  & SSIM  & $p$-value & PSNR  & SSIM  & $p$-value \\
    \midrule
    DIP   &   33.84 & 0.9483      &   9.5045e-08    & 34.49  & 0.9526  & 0.0018  \\
    Deepfill2 & 35.70  & 0.9746  & 0.0109  & 35.56  & 0.9765  & 0.0496  \\
    CTSDG & 35.12  & 0.9634  & 0.0015  & 34.63  & 0.9688  & 0.0111  \\
    RwM-WLRD & \textbf{37.26 } & \textbf{0.9829 } & -     & \textbf{37.71 } & \textbf{0.9857 } & - \\
    \bottomrule
    \end{tabular}%
  \label{tab-comparison-deeplearning}%
\end{table}%
\begin{figure*}
\centering
\newcommand{\fs}{0.18}
\newcommand{\figmargin}{-3mm}
\includegraphics[width=0.96\linewidth]{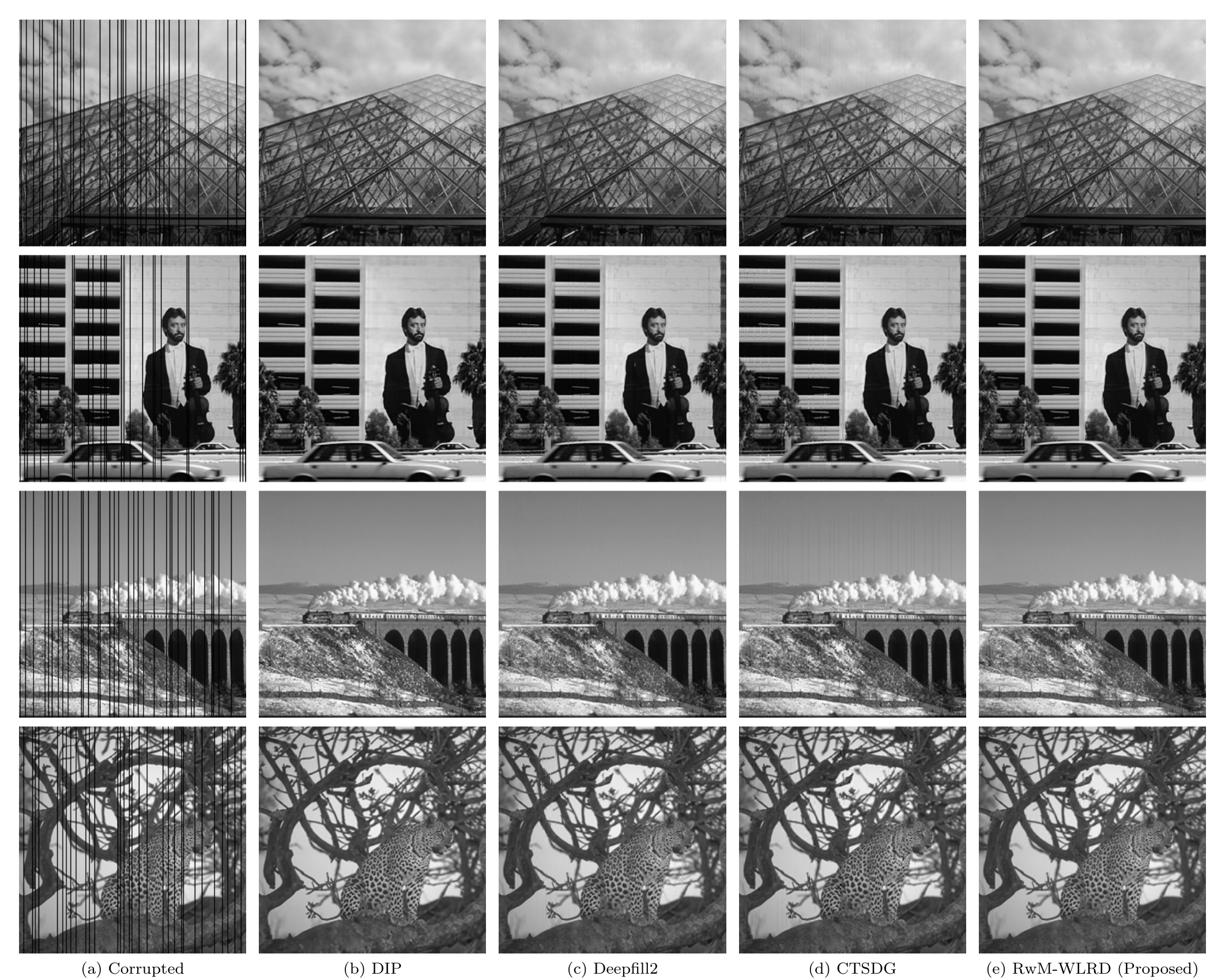}
\caption{Comparison with deep learning-based methods of nature image inpainting in dataset BSD68.}\label{fig-lineresult-dl-bsd68}
\end{figure*}
\subsection{Ablation study}
\textbf{Stability of iterative inpainting}: RwM-WLRD iteratively processes the reconstructed image $\hat{\bm{I}}$(see \TheFig{fig-flowchart} and $\hat{\bm{I}}$ is set to the input image $\bm{I}$ for the first round) to improve inpainting quality. To investigate the stability of iterative processing and the total number of iterations to converge, we use dataset BM3D \cite{BM3D} with some columns missing randomly for each image. The inpainting procedure stops when the difference between consecutive PSNR values is less than 1\%. The iteration-PSNR curves of all images are plotted in \TheFig{iteration-psnr-curves}. One can observe that the PSNR value increases rapidly in the second iteration and then maintains a stable state or keeps increasing slightly, and it converges within 4.67 iterations on average.
\begin{figure}[!htbp]
  \centering
  \includegraphics[width=0.5\linewidth]{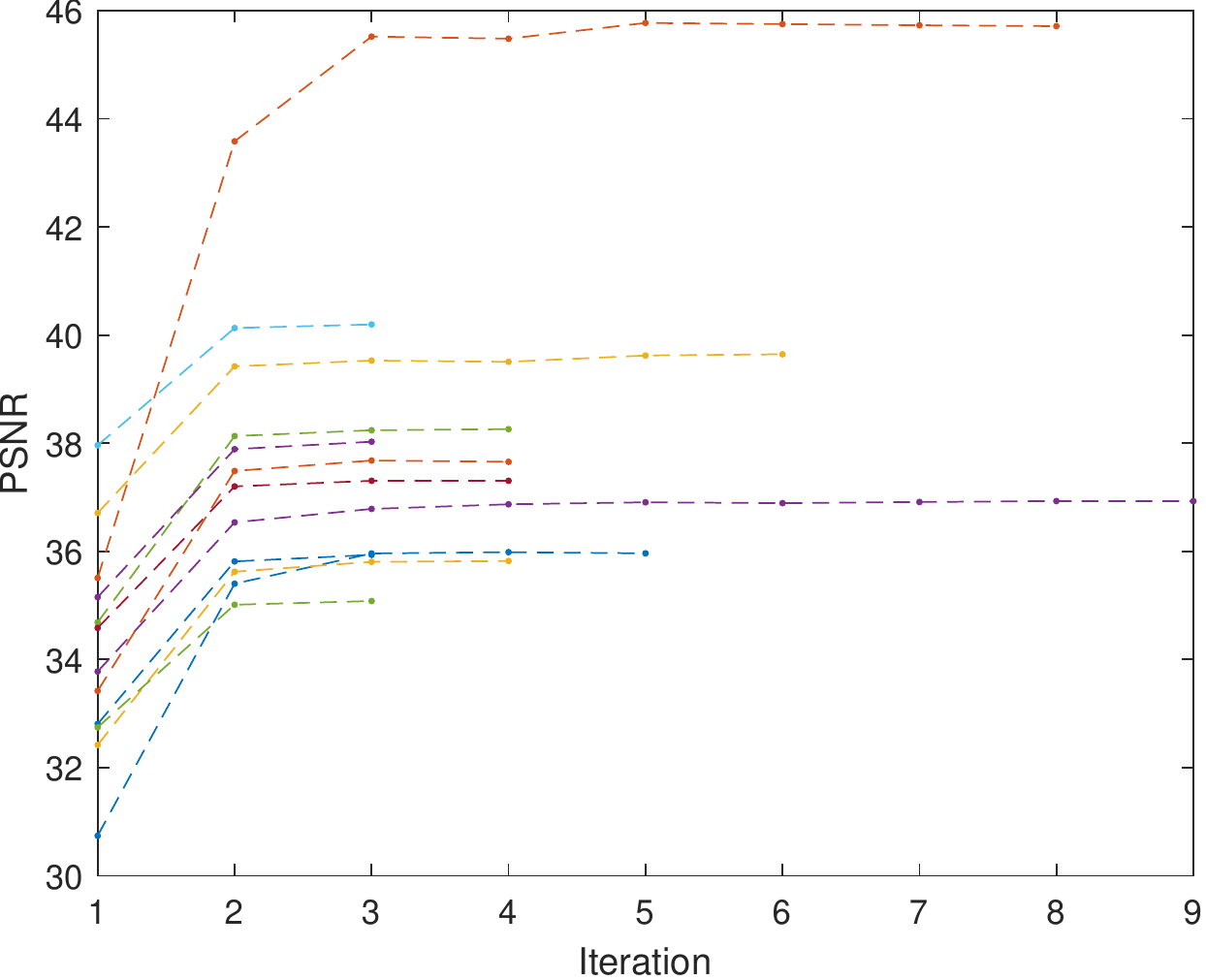}
  \caption{Stability of iterative inpainting. Image dataset BM3D \cite{BM3D} is used. The inpainting procedure stops when the difference between consecutive PSNR values is less than 1\%. }\label{iteration-psnr-curves}
\end{figure}

\textbf{Parameter analysis}:
Our inpainting method involves several parameters, including the number of subregions $\regionnumber$, search radius $r$, patch size $\sqrt{m}$, and the model parameter $\eta$ used in \eqref{eq-adaptive-scad}. To investigate the influence of these parameters, we select 14 images from BSD68 for testing. We first inspect the impact of $\regionnumber$ by fixing other parameters and varying $\regionnumber$ from 20 to 120 with step 10. \TheFig{subfig-ablation-numregions} plots the average PSNR curve. It is clear that the average PSNR curve converges to a stable state when $\regionnumber\geq 60$. Although the average PSNR with $\regionnumber=120$ is slightly better than $\regionnumber=60$, the running time is much larger than that with $\regionnumber=60$ since the computational complexity is quadratic to $\regionnumber$. Therefore we set $\regionnumber=60$ to balance performance and speed. We then study the influence of $r$ with fixed $\regionnumber=60$, and the result is presented in \TheFig{subfig-ablation-radius}. The influence of $r$ on average PSNR is very little though the average PSNR keeps increasing. Therefore, the proposed algorithm is robust to $r$. As for patch size, we vary $\sqrt{m}$ from 4 to 12 with fixed $\regionnumber=60,r=120$. As shown in \TheFig{subfig-ablation-patsize}, the optimal patch size is $\sqrt{m}=5$. However, with $\sqrt{m}=5$, the total number of iterations needed to converge is much larger than that with $\sqrt{m}=8$. On the other hand, if $\sqrt{m}$ is too small such that a target patch is contained entirely in the missing area, then the missing region will not be reconstructed well. Hence a patch size of $8$ is still recommended in practice. \TheFig{subfig-ablation-adp} shows the PSNR results with different values of $\eta$. It turns out that $\eta=0.1$ makes the model performs best, and as $\eta$ keeps increasing, the performance drops slightly. Our analysis shows that the proposed RwM-WLRD is robust to a wide range of parameters. On the other hand, as mentioned before, the proposed matrix completion model \eqref{eq-model} is adaptive to the missing rate $\alpha$, we find that the PSNR score fluctuates if we drop the term $\alpha$. Hence the adaptive design of a regularization function is significant for model stability and robustness.

\begin{figure*}
\centering
\newcommand{\fs}{0.235}
\newcommand{\figmargin}{-3mm}
\subfloat[]{\label{subfig-ablation-numregions}
\includegraphics[width=\fs\linewidth]{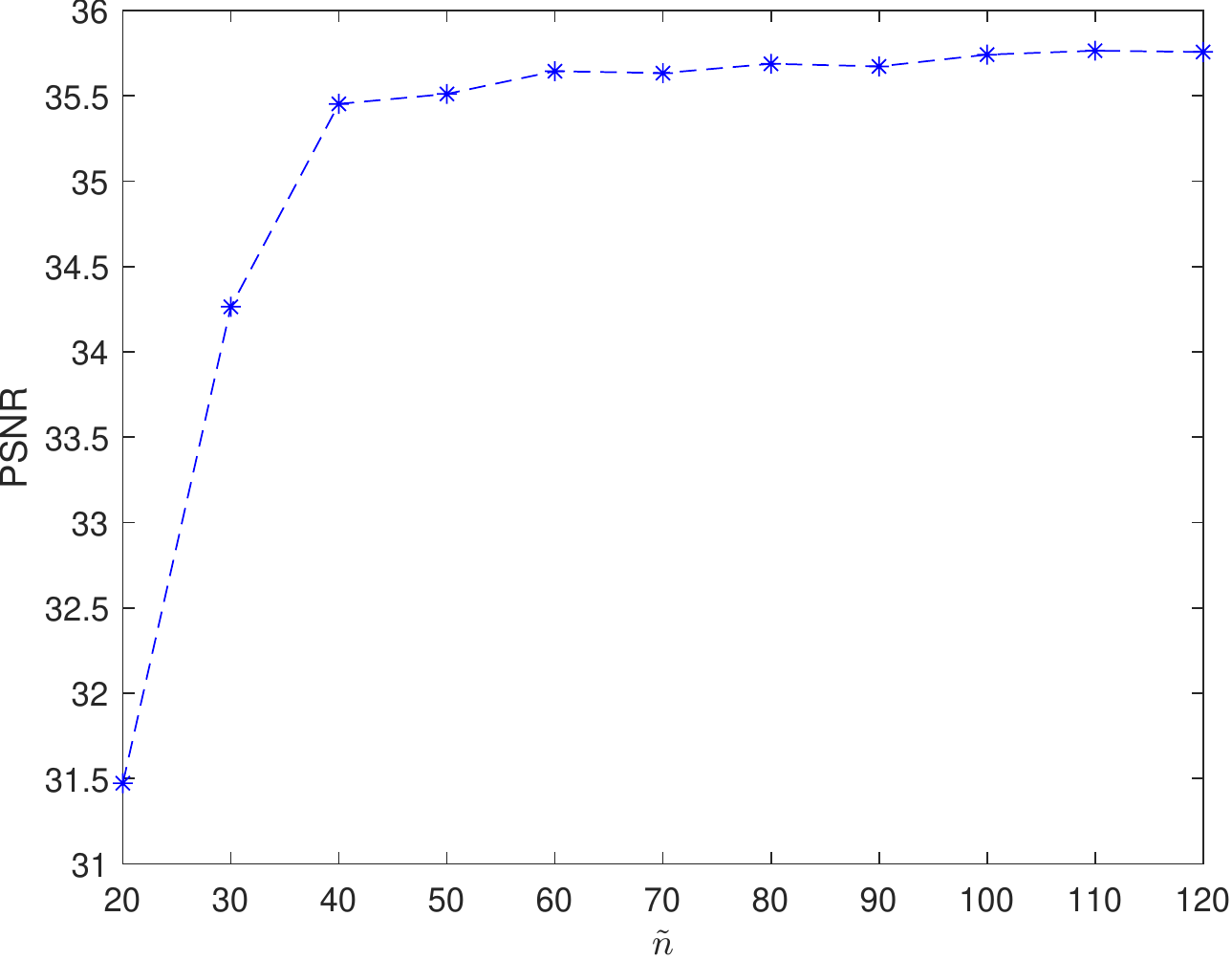}
}
\subfloat[]{\label{subfig-ablation-radius}
\includegraphics[width=\fs\linewidth]{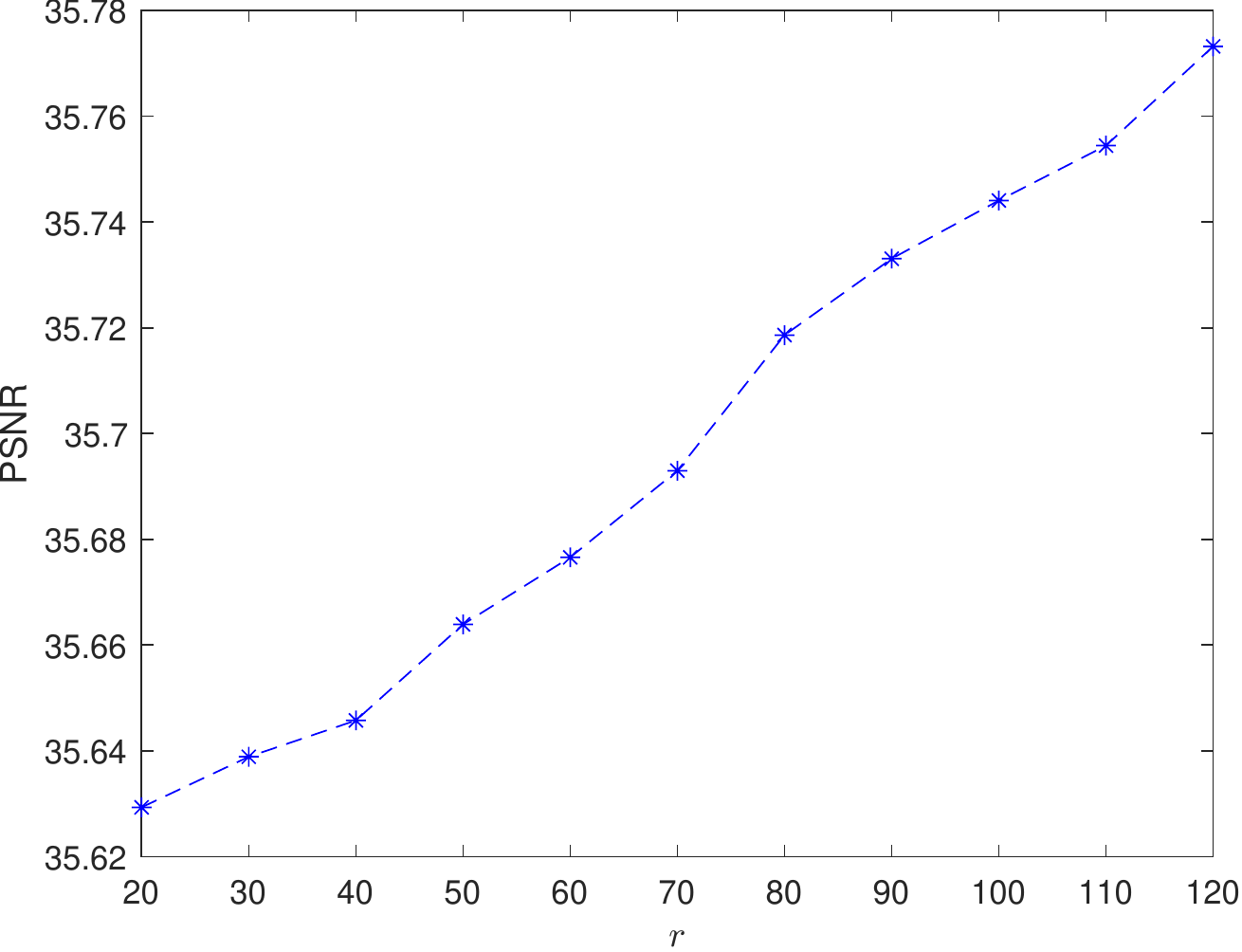}
}
\subfloat[]{\label{subfig-ablation-patsize}
\includegraphics[width=\fs\linewidth]{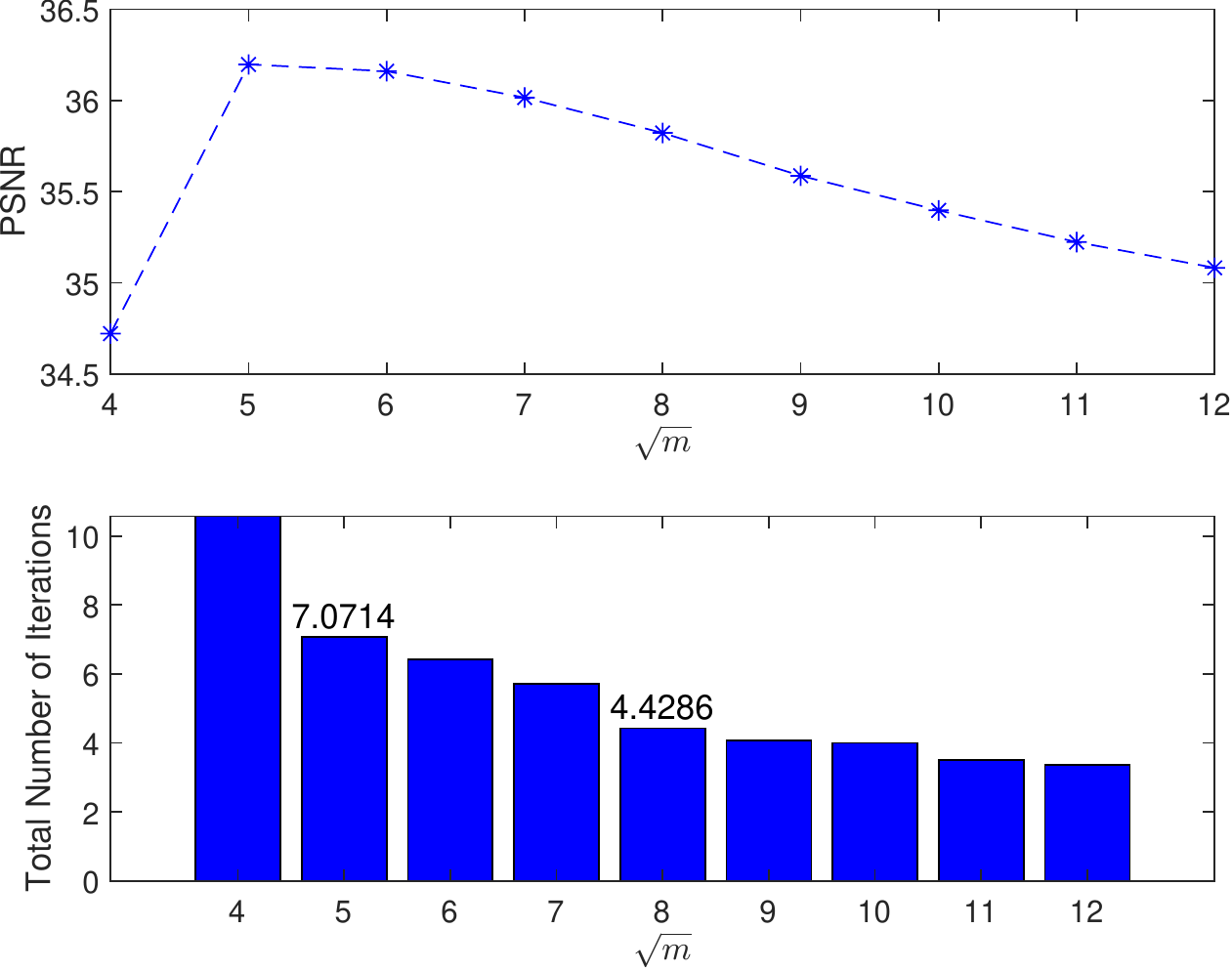}
}
\subfloat[]{\label{subfig-ablation-adp}
\includegraphics[width=\fs\linewidth]{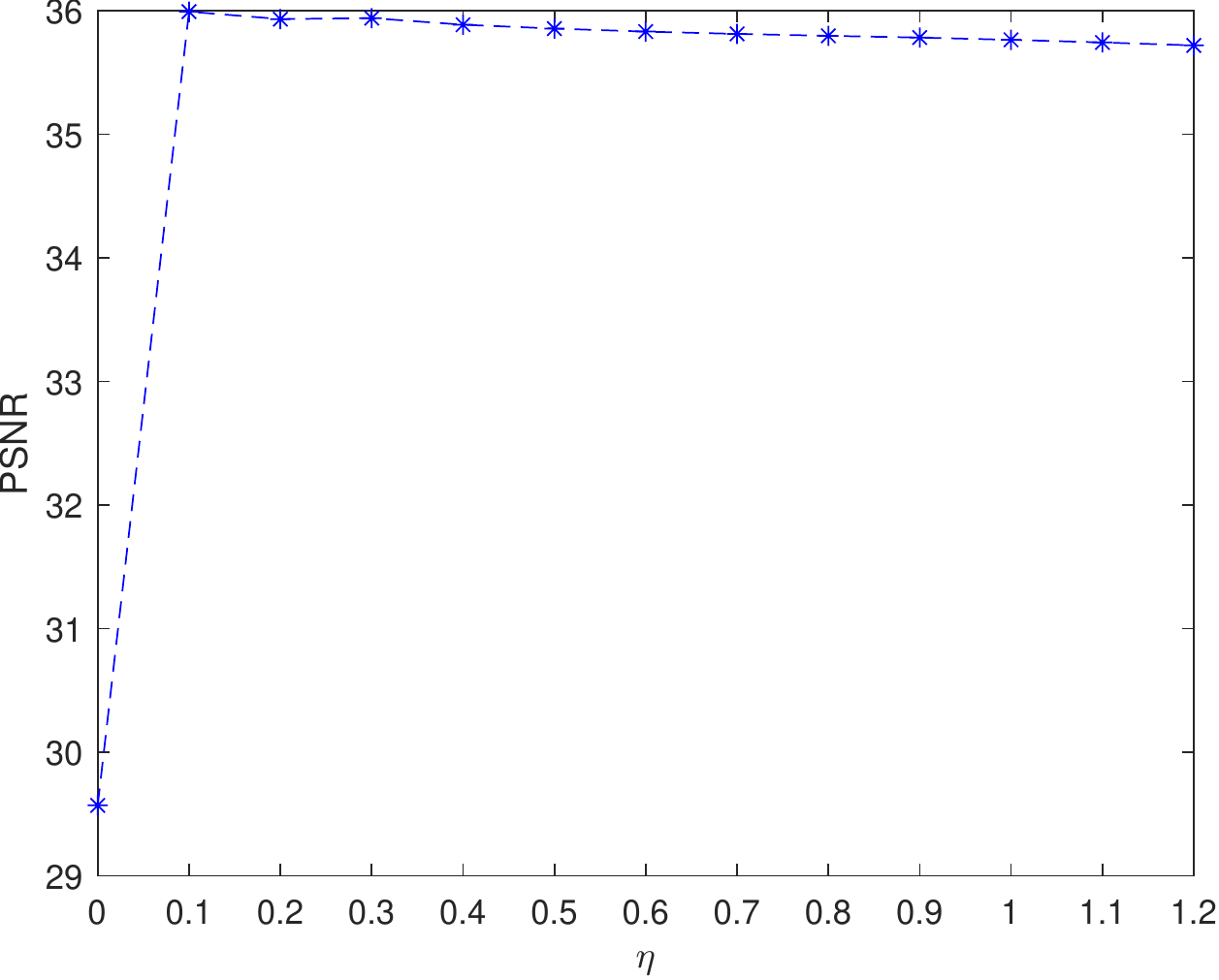}
}
\caption{Influence of the number of subregions $\regionnumber$, search radius $r$, patch size $\sqrt{m}$, and model parameter $\eta$.}\label{fig-ablation}
\end{figure*}

\textbf{Partition of neighborhood}: The core idea of the proposed region-wise matching is to divide the neighbor area of a target patch into multiple subregions and search a given number of similar patches in each subregion separately. We have studied partitioning the $r-$neighbor of a target patch into $\regionnumber$ sectors (see \eqref{eq-dir-points}) in the previous section. Another feasible strategy is to divide the square neighborhood into equal-size grids and search for one similar patch within each grid. \TheFig{fig-rwm-ways} illustrates the ideas of grids and sectors partition, as well as the traditional exhaustive search method. We discover that the grids-partition is also effective for line inpainting, with a slight performance drop compared to the sectors-partition. \TheTable{tab-partition-comparison} summarizes the line inpainting results with different partition approaches of neighborhood.
As illustrated, the sectors-partition method has better performance. We infer the main reason is that the sectors-partition approach works more like propagating information along multiple directions.

\begin{table}[!htbp]
  \centering
  \caption{Line inpainting results in BSD68 dataset with different partition approaches. Second column (\#subregion): number of subregions, third column (\#patch): number of similar patches in each subregion.}
    \begin{tabular}{ccccc}
    \toprule
     Partition Method & \#subregion &   \#patch    & \multicolumn{1}{l}{PSNR} & \multicolumn{1}{l}{SSIM} \\
    \midrule
    None-Partition & 1 & $\regionnumber$   & 20.92 & 0.7292 \\
    Grids-Partition & $\regionnumber$ & 1  & 35.81 & 0.9792 \\
    Sectors-Partition & $\regionnumber$ & 1 & 37.26 & 0.9829 \\
    \bottomrule
    \end{tabular}%
  \label{tab-partition-comparison}%
\end{table}%

\subsection{Extension to blind inpainting}
The proposed weighted low-rank decomposition model \eqref{eq-model} uses a weight matrix $\lambda\MO$ to regularize the sparse component $\MS$. In previous inpainting experiments, we set $\lambda=1$ and achieve promising results. In this part, we show that the proposed model is able to complete missing pixels blindly and remove impulse noise simultaneously by tuning $\lambda$ only, without detection of impulse.

Suppose a grayscale image $\bm{I}\in \realR^{H\times W}$ is corrupted by impulse noise, then we set $\lambda = 1/\sqrt{\max(H,W)}$ to guide the proposed model \eqref{eq-model} to simultaneously remove impulse noise and complete missing pixels.

We compare the proposed RwM-WLRD using the above setting with KALS\cite{doi:10.1137/110843642}\footnote{https://www-users.cse.umn.edu/\~{}lerman/kals/kals\_publish.zip} and median filter based inpainting. Three popular images are selected for demonstration: \emph{Boat}, \emph{House}, and grayscale \emph{Monarch}. For the \emph{Boat} image, $20\%$ of pixels are randomly selected as missing, however the random mask is not provided to the algorithms, thus it is a blind inpainting task. Salt \& pepper noise with density $0.2$ is added to the \emph{House} image using MATLAB function \textit{imnoise}. To demonstrate the proposed RwM-WLRD is able to remove impulse noise simultaneously, $20\%$ of pixels of the grayscale \emph{Monarch} image are damaged with random-valued impulse noise uniformly distributed in $[0,255]$. Besides random-valued impulse noise, we also create a line mask randomly and multiply it with the noisy \emph{Monarch} image. The line mask is provided to all algorithms.
Results are shown in \TheFig{fig-extension-results}. As presented, the proposed RwM-WLRD with weight $\lambda = 1/\sqrt{\max(H,W)}$ is effective for blind inpainting and achieves the best visual quality. By tuning the weight $\lambda$ only, RwM-WLRD removes impulse noise successfully, especially for random-valued impulse noise.

\begin{figure*}[t]
  \centering
  \includegraphics[width=0.96\linewidth]{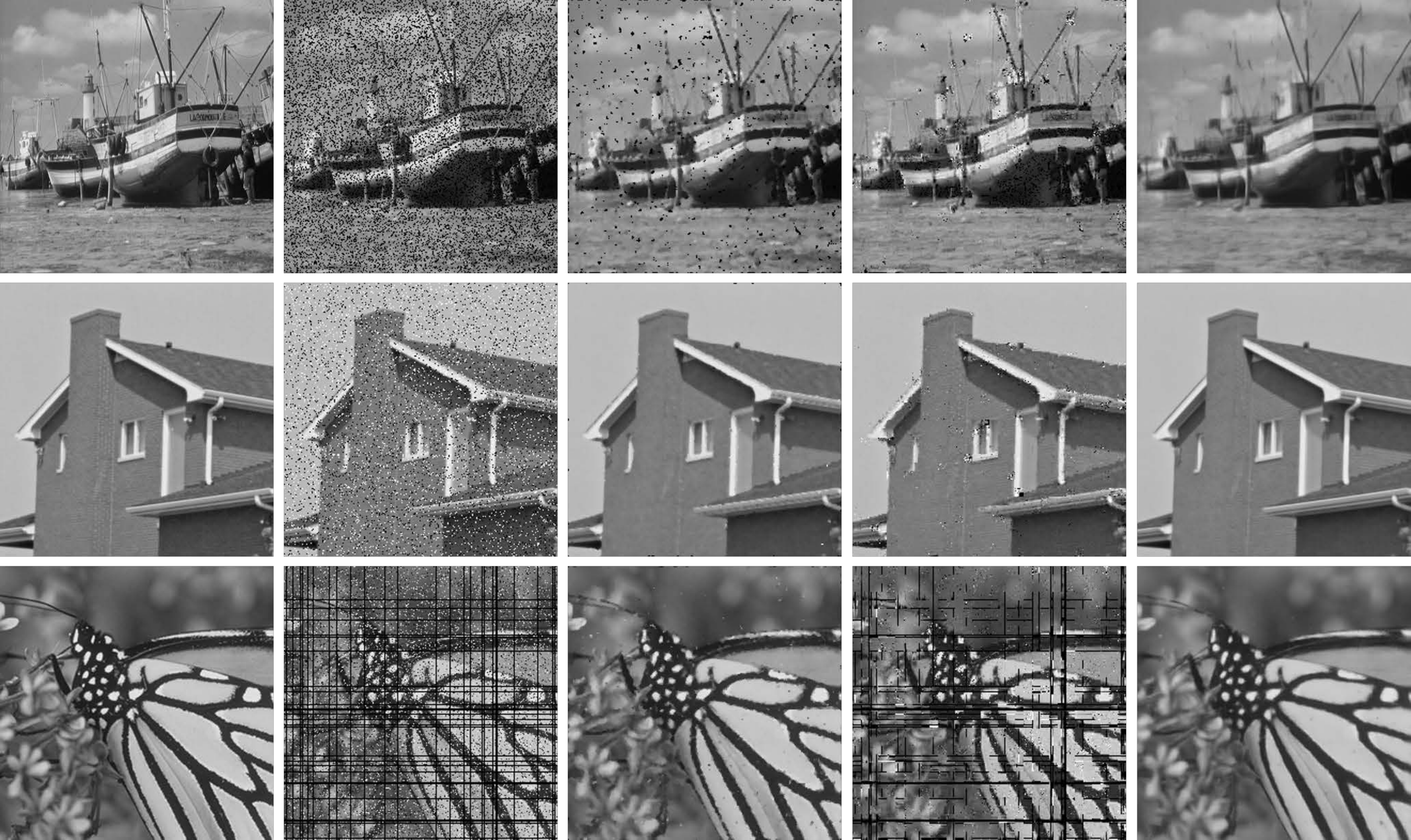}
  \caption{Extension to blind inpainting and impulse noise removal.
  First row: blind inpainting with $20\%$ random missing; Second row: salt \& pepper noise with density $0.2$; Bottom row: random-valued impulse noise (20\%) and entire-row (-column) missing. First column: original image; Second column: damaged image; Third column: median filter (PSNR=25.93, 31.19, 25.30); Fourth column: KALS (PSNR=13.09, 20.15, 12.73); Last column: the porposed RwM-WLRD with $\lambda=\frac{1}{\sqrt{\max(H,W)}}$ (PSNR=27.14, 38.29, 26.72).}\label{fig-extension-results}
\end{figure*}

\subsection{Restoration of remote sensing image}
In this part, we test the proposed algorithm on repairing real remote sensing images degraded by dead lines and stripes. In addition, we compare with stripe removal methods UTV \cite{Bouali2011} and TSWEU \cite{Chang2020}.
As discussed above, the proposed RwM-WLRD removes random noise simultaneously if the weight parameter $\lambda$ is set properly. Now we show that without grouping similar patches, the proposed matrix completion model NC-WLRD is suitable for destriping of remote sensing images.
\TheFig{fig-destripe} illustrates the effectiveness of destriping. The proposed NC-WLRD is competitive to the trained destriping CNN model TSWEU \cite{Chang2020}. Note that for the task of destriping, the stripes component is represented by the low rank part $\ML$ in \eqref{eq-model}, and the output image is $\MS$.

The inpainting results of horizontal dead lines are shown in \TheFig{fig-rsi}. Note that UTV \cite{Bouali2011} and TSWEU \cite{Chang2020} are two destriping algorithms, the missing mask (i.e. position of dead lines) is not required. For other inpainting algorithms, the missing mask is provided by the detection result of NC-WLRD. The proposed NC-WLRD is able to output accurate positions of dead lines (low rank part), but the destriping result in this case is not as satisfactory as in \TheFig{fig-destripe}. Thus we follow the detection-inpainting pipeline and apply the proposed RwM-WLRD to repair all corrupted lines. As shown in \TheFig{fig-rsi}, RwM-WLRD successfully restores the remote sensing image degraded by dead lines.
\begin{figure*}
\centering
\newcommand{\fs}{0.24}
\newcommand{\figmargin}{-3mm}
\includegraphics[width=0.98\linewidth]{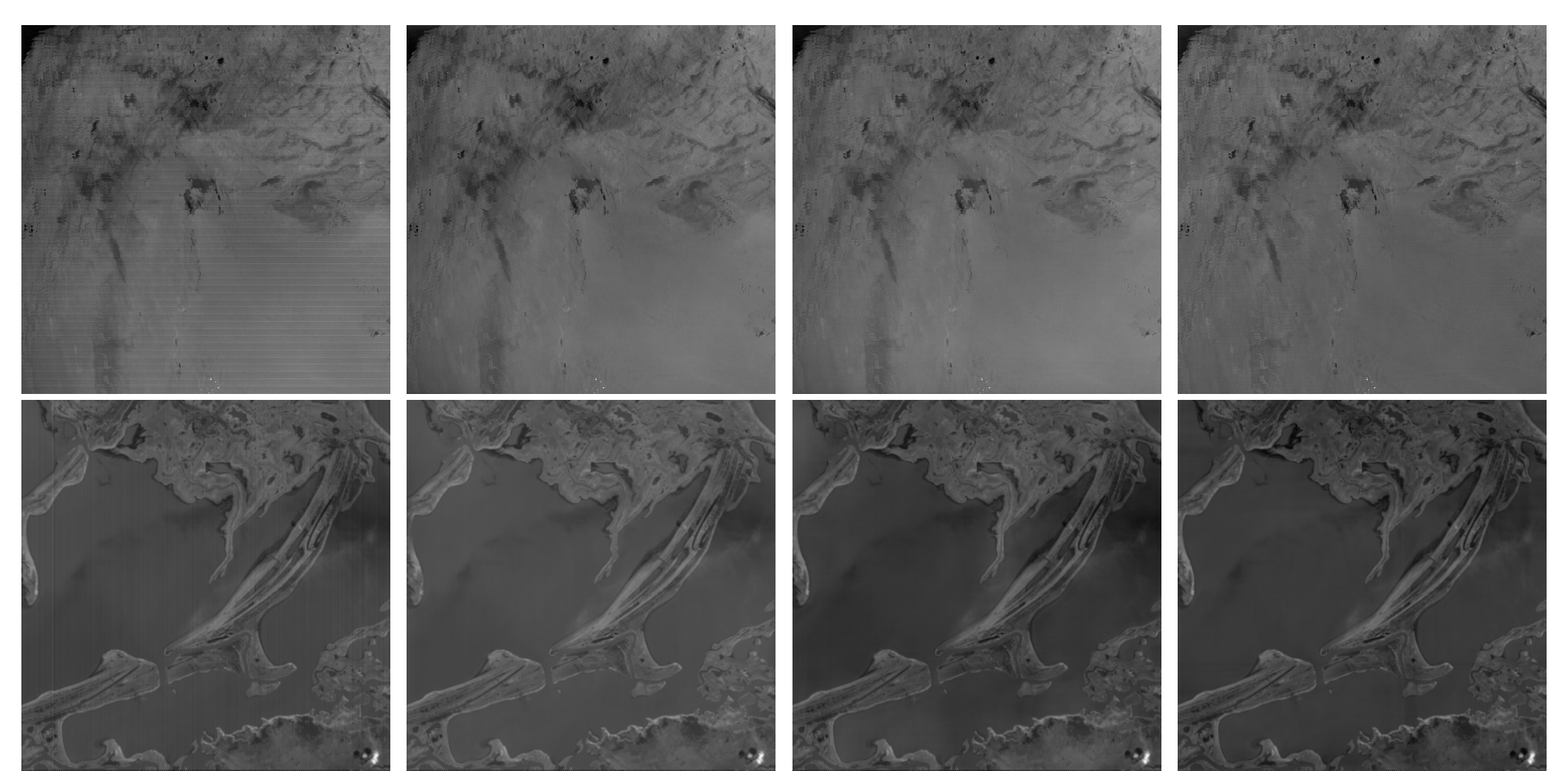}
\caption{Destriping of remote sensing image. We set $\lambda=0.1/\sqrt{\max(H,W)}$ and $\MO=\bm{1}$ in the proposed model NC-WLRD. From left to right: degraded image, UTV \cite{Bouali2011}, TSWEU \cite{Chang2020}, NC-WLRD (proposed).}\label{fig-destripe}
\end{figure*}

\begin{figure*}
\centering
\newcommand{\fs}{0.19}
\newcommand{\figmargin}{-3mm}
\includegraphics[width=0.98\linewidth]{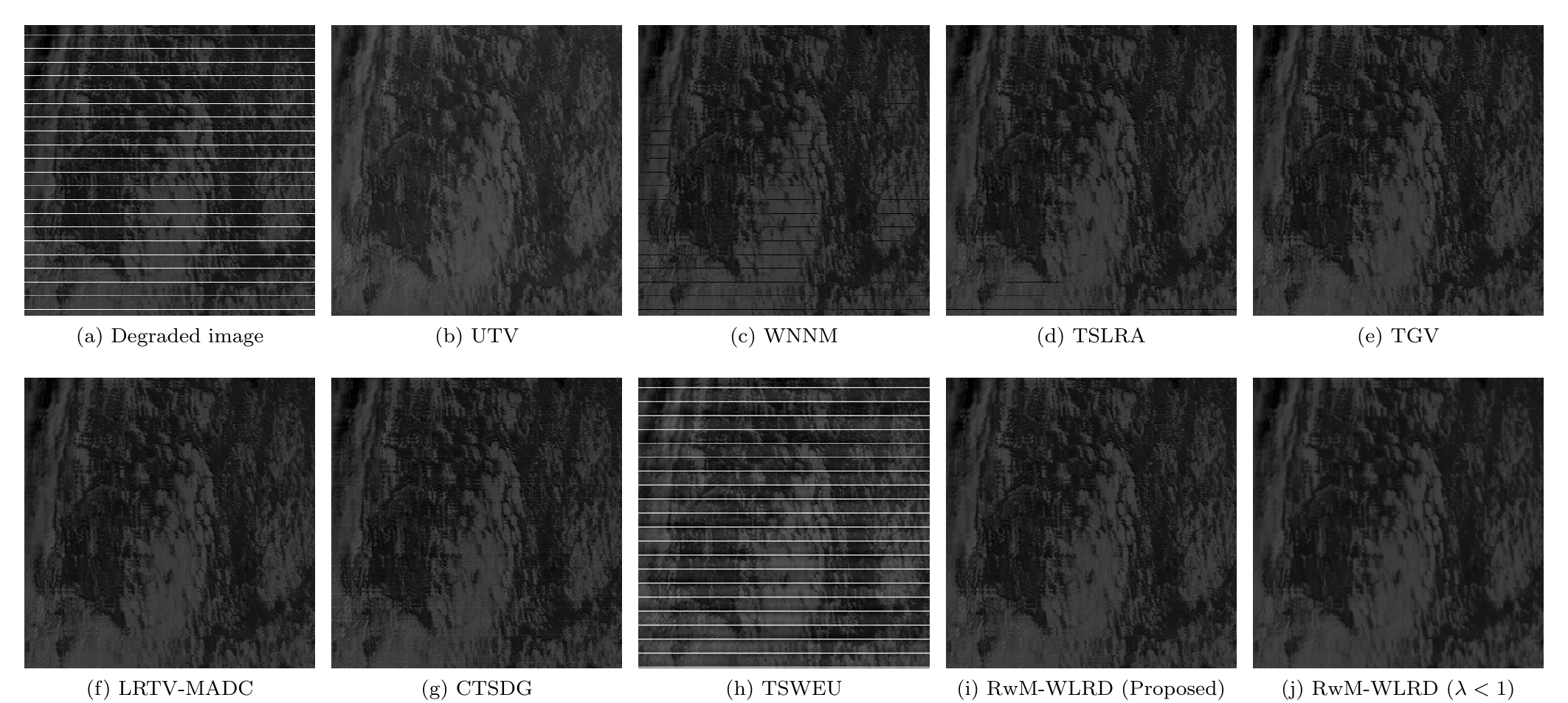}
\caption{Results of remote sensing image inpainting (real case of dead lines). The trained CNN model TSWEU can't repair the image well. For other inpainting methods, the inpainting mask is provided by the detection result of the proposed NC-WLRD. We set $\lambda=1/\sqrt{\max(H,W)}$ in the proposed simultaneously denoising version of RwM-WLRD ($\lambda<1$).}\label{fig-rsi}
\end{figure*}

\subsection{Limitations and further improvements}
The proposed region-wise matching (RwM) approach solves the problem of line inpainting, while previous low-rank-based SOTA methods remain artifacts when inpainting entire missing columns. RwM is also effective for repairing scratches and removing overlaid text.
RwM works well for these structurally-missing regions (line, scratch, and text), however, for random missing mask, we find that traditional block matching method (non-partition, see \TheFig{fig-rwm-ways}) generates better results in some examples. Therefore, the proposed RwM is less competitive for inpainting random missing pixels. One possible technique to address this problem is to combine RwM with traditional block matching method, or to use a smarter partition of neighbor area.

On the other hand, the proposed RwM-WLRD has difficulty repairing large holes. The iterative diffusion strategy used in TSLRA \cite{guo2017patch} makes it possible to fill large blocks ($15\times15$) in texture regions. Although our method performs well with missing holes of larger size (say $20\times20$), it relies heavily on finding similar patterns around the hole.
To solve this limitation, we intend to apply the proposed region-wise matching method to search for similar patches in a large dataset.
\section{Conclusions}
\label{section-conclusions}
This paper proposes a non-convex weighted low-rank decomposition model NC-WLRD for matrix completion and optimizes the model by using the alternating direction method of multipliers (ADMM). 
Compared to several state-of-the-art matrix completion methods, NC-WLRD has a competitive convergence speed and achieves the best image completion results.
We also propose a region-wise matching strategy to tackle the challenge that most of traditional low-rank based inpainting methods fail to fill entirely-missing rows (columns). The core idea of region-wise matching is to divide the neighbor area of a target patch into multiple subregions and search a given number of similar patches in each subregion separately. By using the region-wise matching strategy, the proposed RwM-WLRD restores entirely-missing rows (columns) robustly. Extensive
experiments on line inpainting demonstrate the effectiveness and superiority of RwM-WLRD over other competitive inpainting algorithms. By changing only the weight parameter $\lambda$, RwM-WLRD successfully removes  random-valued impulse noise simultaneously and is effective for blind inpainting. Experiments on the restoration of remote sensing images demonstrate that the proposed model has competitive performance for destriping (stripe noise removal) and repairing dead lines.
\section*{Acknowledgments}
The research has been supported in part by the National Natural Science Foundation of China (12071263, 11971269, 61671276) and Natural Science Foundation of Shandong Province (ZR2019MF045). The source code can be found at \href {https://doi.org/10.17632/xvhvwhtbj3.1}
  {\path{doi: 10.17632/xvhvwhtbj3.1}}
\bibliographystyle{ieeetr}
\bibliography{rpcabibs}
\clearpage
\onecolumn 
\appendix
\section{Convergence analysis}
\label{appendix}
\numberwithin{equation}{section}
\counterwithin{equation}{section}
\setcounter{equation}{0}
In this section, we give the convergence analysis of Algorithm \ref{alg-admm-rpca}. We first prove that the generated matrix sequences $\{\ML^{(k)}\},\{\MS^{(k)}\},\{A^{(k)}\}$ are bounded.
\newtheorem{lemma}{Lemma}[section]
\begin{lemma}\label{lemma1}
$\{\MATRIX{A}^{(k+1)} \}$ is bounded.
\end{lemma}
\begin{proof}
  Let $\MATRIX{U}\MATRIX{S}\MATRIX{V}^{T}$ be the SVD of $\MATRIX{D}^{(k)}=\frac{1}{\mu^{(k)}}\MATRIX{A}^{(k)}+\MY-\MS^{(k+1)}$, where $\MATRIX{S}=diag(s_1,\cdots,s_n)$ is the diagonal singular value matrix.
We have proved that in the $(k+1)$-th iteration of Algorithm 1, the updating of $\ML^{(k+1)}$ follows
\begin{equation}\label{xk1}
  \ML^{(k+1)} = \MATRIX{U}\MATRIX{\Sigma}\MATRIX{V}^{T},
\end{equation}
where $\MATRIX{\Sigma}=diag(\sigma_1,\cdots,\sigma_n)$, and the calculation of $\sigma_i$ is given by
\begin{equation}\label{eq-shrinkage-sv-solution-a}
  \sigma_{i} = \begin{cases}
                 \max(\frac{\mu^{(k)}\cdot s_i-1}{\mu^{(k)}},0), & \mbox{if } \ s_i\leq \frac{\mu^{(k)}+1}{\mu^{(k)}} \\
                 \frac{\mu^{(k)}\cdot s_i-\frac{\gamma}{\gamma-1}}
                 {\mu^{(k)}-\frac{1}{\gamma-1}}, & \mbox{if } \frac{\mu^{(k)}+1}{\mu^{(k)}}<s_i<\gamma \\
                 s_i, & \mbox{otherwise}.
               \end{cases}
\end{equation}
Now we show that the generated sequence of the Lagrange multiplier matrix $\{\MATRIX{A}^{(k)}\}$ is bounded. Since $\MATRIX{A}^{(k+1)}=\MATRIX{A}^{(k)}+\mu^{(k)}(\MY-\MS^{(k+1)}-\ML^{(k+1)})$, we have
\begin{equation}\label{eq-Akplus1}
\begin{aligned}
    \| \MATRIX{A}^{(k+1)} \|_F  &= \|  \MATRIX{A}^{(k)}+\mu^{(k)}(\MY-\MS^{(k+1)}-\ML^{(k+1)})  \|_F \\
                                &= \mu^{(k)} \| \frac{1}{\mu^{(k)}}\MATRIX{A}^{(k)}+\MY-\MS^{(k+1)} - \ML^{(k+1)} \|_F \\
                                &= \mu^{(k)} \| \MATRIX{U}\MATRIX{S}\MATRIX{V}^{T} - \MATRIX{U}\MATRIX{\Sigma}\MATRIX{V}^{T} \|_F \\
                                &= \mu^{(k)} \| \MATRIX{S} - \MATRIX{\Sigma} \|_F \\
                                &= \mu^{(k)} \sqrt{\sum_{s_i\leq 1+\frac{1}{\mu^{(k)}}}\left(\max(\frac{\mu^{(k)}\cdot s_i-1}{\mu^{(k)}},0)-s_i\right)^2 + \sum_{\frac{\mu^{(k)}+1}{\mu^{(k)}}<s_i<\gamma}\left(\frac{\mu^{(k)}\cdot s_i-\frac{\gamma}{\gamma-1}}{\mu^{(k)}-\frac{1}{\gamma-1}}-s_i\right)^2}\\
                                &\leq \mu^{(k)}\sqrt{\sum_{s_i\leq 1+\frac{1}{\mu^{(k)}}}(\frac{1}{\mu^{(k)}})^2 + \sum_{\frac{\mu^{(k)}+1}{\mu^{(k)}}<s_i<\gamma}\left(\frac{\frac{s_i-\gamma}{\gamma-1}}{\mu^{(k)}-\frac{1}{\gamma-1}}\right)^2}\\
                                &\leq  \mu^{(k)}\sqrt{\sum_{s_i\leq 1+\frac{1}{\mu^{(k)}}}(\frac{1}{\mu^{(k)}})^2 + \sum_{\frac{\mu^{(k)}+1}{\mu^{(k)}}<s_i<\gamma}\frac{(\gamma-1-\frac{1}{\mu^{(k)}})^2}{(\mu^{(k)}(\gamma-1)-1)^2}}\\
                                &= \mu^{(k)}\sqrt{\sum_{s_i\leq 1+\frac{1}{\mu^{(k)}}}(\frac{1}{\mu^{(k)}})^2 + \sum_{\frac{\mu^{(k)}+1}{\mu^{(k)}}<s_i<\gamma}(\frac{1}{\mu^{(k)}})^2}=\sqrt{n}.\\
\end{aligned}
\end{equation}
Thus $\{\MATRIX{A}^{(k+1)} \}$ is bounded.
\end{proof}

\begin{lemma}
$\{\ML^{(k+1)} \},\{\MS^{(k+1)} \}$ are bounded.
\end{lemma}
\begin{proof}
We first show that the augmented lagrange function $\Gamma$ is bounded in the generated sequences, i.e. $\Gamma(\ML^{(k+1)},\MS^{(k+1)},\MATRIX{A}^{(k)},\mu^{(k)})<\infty$.
  In each step $k+1$, since $\ML^{(k+1)}$ and $\MS^{(k+1)}$ are the optimal solutions of the $\ML$ subproblem and $\MS$ subproblem respectively, we have
\begin{equation*}
\begin{aligned}
  &\Gamma(\ML^{(k+1)},\MS^{(k+1)},\MATRIX{A}^{(k)},\mu^{(k)})\\
  &=\|\ML^{(k+1)}\|_{\varphi,*}+\lambda\|\MO\odot\MS^{(k+1)}\|_{1} +\langle\MATRIX{A}^{(k)},\MY-\ML^{(k+1)}-\MS^{(k+1)}\rangle +\frac{\mu^{(k)}}{2}\|\MY-\ML^{(k+1)}-\MS^{(k+1)}\|_{F}^{2}\\
  &\leq \|\ML^{(k)}\|_{\varphi,*}+\lambda\|\MO\odot\MS^{(k)}\|_{1} +\langle\MATRIX{A}^{(k)},\MY-\ML^{(k)}-\MS^{(k)}\rangle +\frac{\mu^{(k)}}{2}\|\MY-\ML^{(k)}-\MS^{(k)}\|_{F}^{2}\\
\end{aligned}
\end{equation*}
Substitute $\MATRIX{A}^{(k)}=\MATRIX{A}^{(k-1)}+\mu^{(k-1)}(\MY-\ML^{(k)}-\MS^{(k)})$ into the above formula we have
\begin{equation*}
\begin{aligned}
  &  \|\ML^{(k)}\|_{\varphi,*}+\lambda\|\MO\odot\MS^{(k)}\|_{1} +\langle\MATRIX{A}^{(k-1)}+\mu^{(k-1)}(\MY-\ML^{(k)}-\MS^{(k)}),\MY-\ML^{(k)}-\MS^{(k)}\rangle +\frac{\mu^{(k)}}{2}\|\MY-\ML^{(k)}-\MS^{(k)}\|_{F}^{2}\\
  & = \|\ML^{(k)}\|_{\varphi,*}+\lambda\|\MO\odot\MS^{(k)}\|_{1} +\langle\MATRIX{A}^{(k-1)},\MY-\ML^{(k)}-\MS^{(k)}\rangle +\frac{\mu^{(k-1)}}{2}\|\MY-\ML^{(k)}-\MS^{(k)}\|_{F}^{2}\\
  & + \frac{\mu^{(k)}-\mu^{(k-1)}}{2}\|\MY-\ML^{(k)}-\MS^{(k)}\|_{F}^{2} + \langle\mu^{(k-1)}(\MY-\ML^{(k)}-\MS^{(k)}),\MY-\ML^{(k)}-\MS^{(k)}\rangle\\
  & = \Gamma(\ML^{(k)},\MS^{(k)},\MATRIX{A}^{(k-1)},\mu^{(k-1)}) + \frac{\mu^{(k)}+\mu^{(k-1)}}{2}\|\MY-\ML^{(k)}-\MS^{(k)}\|_{F}^{2} \\
\end{aligned}
\end{equation*}

Note that $\MY-\ML^{(k)}-\MS^{(k)} = \frac{1}{\mu^{(k-1)}}(\MATRIX{A}^{(k)}-\MATRIX{A}^{(k-1)})$, hence
\begin{equation*}
\begin{aligned}
  &\Gamma(\ML^{(k+1)},\MS^{(k+1)},\MATRIX{A}^{(k)},\mu^{(k)})\\
  & = \Gamma(\ML^{(k)},\MS^{(k)},\MATRIX{A}^{(k-1)},\mu^{(k-1)}) + \frac{\mu^{(k)}+\mu^{(k-1)}}{2(\mu^{(k-1)})^2}\|\MATRIX{A}^{(k)}-\MATRIX{A}^{(k-1)}\|_{F}^{2}\\
  & = \Gamma(\ML^{(1)},\MS^{(1)},\MATRIX{A}^{(0)},\mu^{(0)}) + \sum_{j=1}^{k}\frac{\mu^{(j)}+\mu^{(j-1)}}{2(\mu^{(j-1)})^2}\|\MATRIX{A}^{(j)}-\MATRIX{A}^{(j-1)}\|_{F}^{2}\\
\end{aligned}
\end{equation*}
Because $\| \MATRIX{A}^{(j)} \|_F\leq \sqrt{n}$ holds for any $j$, and $\|\MATRIX{A}^{(j)}-\MATRIX{A}^{(j-1)}\|_{F}^{2}\leq 2n$, it follows that
\begin{equation}
  \Gamma(\ML^{(k+1)},\MS^{(k+1)},\MATRIX{A}^{(k)},\mu^{(k)})\leq  \Gamma(\ML^{(1)},\MS^{(1)},\MATRIX{A}^{(0)},\mu^{(0)}) + 2n\sum_{j=1}^{\infty}\frac{\mu^{(j)}+\mu^{(j-1)}}{2(\mu^{(j-1)})^2}
\end{equation}
Since $\mu^{(k+1)}=\rho\mu^{(k)}$ and $\rho>1$,
\begin{equation}
  \sum_{j=1}^{\infty}\frac{\mu^{(j)}+\mu^{(j-1)}}{2(\mu^{(j-1)})^2}<\sum_{j=1}^{\infty}\frac{\rho^j \mu^{(0)}}{(\rho^{j-1}\mu^{(0)})^2}<\infty.
\end{equation}
Therefore $\Gamma(\ML^{(k+1)},\MS^{(k+1)},\MATRIX{A}^{(k)},\mu^{(k)})$ is bounded. To prove $\{\ML^{(k+1)} \},\{\MS^{(k+1)} \}$ are bounded, move the first two term of $\Gamma(\ML^{(k+1)},\MS^{(k+1)},\MATRIX{A}^{(k)},\mu^{(k)})$ to one side, we obtain
\begin{equation*}
\begin{aligned}
&\|\ML^{(k+1)}\|_{\varphi,*}+\lambda\|\MO\odot\MS^{(k+1)}\|_{1} \\
&=  \Gamma(\ML^{(k+1)},\MS^{(k+1)},\MATRIX{A}^{(k)},\mu^{(k)}) - \langle\MATRIX{A}^{(k)},\MY-\ML^{(k+1)}-\MS^{(k+1)}\rangle -\frac{\mu^{(k)}}{2}\|\MY-\ML^{(k+1)}-\MS^{(k+1)}\|_{F}^{2}\\
&= \Gamma(\ML^{(k+1)},\MS^{(k+1)},\MATRIX{A}^{(k)},\mu^{(k)}) - \frac{\mu^{(k)}}{2}(\|\MY-\ML^{(k+1)}-\MS^{(k+1)}+\frac{1}{\mu^{(k)}}\MATRIX{A}^{(k)}\|_{F}^{2}-\|\frac{1}{\mu^{(k)}}\MATRIX{A}^{(k)}  \|_F^2)\\
&= \Gamma(\ML^{(k+1)},\MS^{(k+1)},\MATRIX{A}^{(k)},\mu^{(k)}) - \frac{\mu^{(k)}}{2}(\|\frac{1}{\mu^{(k)}}\MATRIX{A}^{(k+1)}\|_{F}^{2}-\|\frac{1}{\mu^{(k)}}\MATRIX{A}^{(k)}  \|_F^2)\\
&= \Gamma(\ML^{(k+1)},\MS^{(k+1)},\MATRIX{A}^{(k)},\mu^{(k)}) - \frac{1}{2\mu^{(k)}}(\|\MATRIX{A}^{(k+1)}\|_{F}^{2}-\|\MATRIX{A}^{(k)}  \|_F^2)\\
\end{aligned}
\end{equation*}
Use lemma \ref{lemma1}, it is easy to see that $\ML^{(k+1)}$ and $\MS^{(k+1)}$ are bounded.
\end{proof}


\begin{theorem}
The sequences $\{\ML^{(k+1)}\}$ and $\{\MS^{(k+1)}\}$ generated by algorithm 1 satisfy the following property:
\begin{equation}\label{convg}
  \begin{aligned}
  &\mathop{\lim}_{k\to\infty}\|\ML^{(k+1)}- \ML^{(k)}\|_F =0\\
  &\mathop{\lim}_{k\to\infty}\|\MS^{(k+1)}- \MS^{(k)}\|_F =0\\
  \end{aligned}
\end{equation}
\label{thm-1}
\end{theorem}
\begin{proof}
  Let $\MATRIX{E}^{(k)}=\frac{1}{\mu^{(k)}}\MATRIX{A}^{(k)}+\MY-\ML^{(k)}$ and $\bm{\epsilon}=\frac{\lambda}{\mu^{(k)}}\MO$. It is easy to deduce
\begin{equation*}
\begin{aligned}
\MS^{(k)}&=\MY-\ML^{(k)}-\frac{1}{\mu^{(k-1)}}(\MATRIX{A}^{(k)}-\MATRIX{A}^{(k-1)})\\
&=\frac{1}{\mu^{(k)}}\MATRIX{A}^{(k)}+\MY-\ML^{(k)}-\frac{1}{\mu^{(k)}}\MATRIX{A}^{(k)}-\frac{1}{\mu^{(k-1)}}(\MATRIX{A}^{(k)}-\MATRIX{A}^{(k-1)})\\
&=\MATRIX{E}^{(k)}-(\frac{1}{\mu^{(k)}}\MATRIX{A}^{(k)}+\frac{1}{\mu^{(k-1)}}(\MATRIX{A}^{(k)}-\MATRIX{A}^{(k-1)}))\\
\end{aligned}
\end{equation*}
The updating of $\MS^{(k+1)}$ is given in the text as $\MS^{(k+1)}=\MO\odot\shrinkage_{\bm{\epsilon}}(\MATRIX{E}^{(k)}) + (1-\MO)\odot \MATRIX{E}^{(k)}$, thus
\begin{equation*}
\begin{aligned}
 &\mathop{\lim}_{k\to\infty}\|\MS^{(k+1)}- \MS^{(k)}\|_F \\
 &= \mathop{\lim}_{k\to\infty}\|\MO\odot\shrinkage_{\bm{\epsilon}}(\MATRIX{E}^{(k)}) + (1-\MO)\odot \MATRIX{E}^{(k)}  -\MS^{(k)}\|_F\\
 &= \mathop{\lim}_{k\to\infty}\|\MO\odot\shrinkage_{\bm{\epsilon}}(\MATRIX{E}^{(k)}) + (1-\MO)\odot \MATRIX{E}^{(k)}  -\MATRIX{E}^{(k)}-(\frac{1}{\mu^{(k)}}\MATRIX{A}^{(k)}+\frac{1}{\mu^{(k-1)}}(\MATRIX{A}^{(k)}-\MATRIX{A}^{(k-1)}))\|_F\\
 &\leq \mathop{\lim}_{k\to\infty}\left(\|\MO\odot\shrinkage_{\bm{\epsilon}}(\MATRIX{E}^{(k)}) + (1-\MO)\odot \MATRIX{E}^{(k)}  -\MATRIX{E}^{(k)}\|_F + \|(\frac{1}{\mu^{(k)}}\MATRIX{A}^{(k)}+\frac{1}{\mu^{(k-1)}}(\MATRIX{A}^{(k)}-\MATRIX{A}^{(k-1)}))\|_F \right)\\
 &= \mathop{\lim}_{k\to\infty}\|\MO\odot\shrinkage_{\bm{\epsilon}}(\MATRIX{E}^{(k)}) + (1-\MO)\odot \MATRIX{E}^{(k)}  -\MATRIX{E}^{(k)}\|_F  \\
 & = \mathop{\lim}_{k\to\infty}\|\MO\odot\shrinkage_{\bm{\epsilon}}(\MATRIX{E}^{(k)}) - \MO\odot \MATRIX{E}^{(k)} \|_F  \\
 & = \mathop{\lim}_{k\to\infty}\|\MO\odot(\shrinkage_{\bm{\epsilon}}(\MATRIX{E}^{(k)}) - \MATRIX{E}^{(k)} )\|_F  \\
 & \leq \mathop{\lim}_{k\to\infty}\|\shrinkage_{\frac{\lambda}{\mu^{(k)}}}(\MATRIX{E}^{(k)}) - \MATRIX{E}^{(k)} \|_F  \\
 & \leq \mathop{\lim}_{k\to\infty} \sum_{ij}(\sign{\MATRIX{E}^{(k)}_{ij}} \cdot \max(\frac{\lambda}{\mu^{(k)}}-\MATRIX{E}^{(k)}_{ij},0) - \MATRIX{E}^{(k)}_{ij} ) \\
 &\leq \mathop{\lim}_{k\to\infty} \frac{mn\lambda}{\mu^{(k)}}=0
\end{aligned}
\end{equation*}
We have used the boundeness of $\MATRIX{A}^{(k)}$ in above deduction.

Next we prove $\mathop{\lim}_{k\to\infty}\|\ML^{(k+1)}- \ML^{(k)}\|_F=0$. Since $\MATRIX{A}^{(k+1)}=\MATRIX{A}^{(k)}+\mu^{(k)}(\MY-\MS^{(k+1)}-\ML^{(k+1)})$, we have
\begin{equation*}
  \ML^{(k+1)} =  \MY-\MS^{(k+1)} + \frac{1}{\mu^{(k)}}(\MATRIX{A}^{(k)} - \MATRIX{A}^{(k+1)})
\end{equation*}
Let $\MATRIX{U}\MATRIX{S}\MATRIX{V}^{T}$ be the SVD of $\MATRIX{D}^{(k)}=\frac{1}{\mu^{(k-1)}}\MATRIX{A}^{(k-1)}+\MY-\MS^{(k)}$, $\MATRIX{S}=diag(s_1,\cdots,s_n)$, then $\ML^{(k)} = \MATRIX{U}\MATRIX{\Sigma}\MATRIX{V}^{T},\MATRIX{\Sigma}=diag(\sigma_1,\cdots,\sigma_n)$ and $\sigma_i$ is given by
\begin{equation}
  \sigma_{i} = \begin{cases}
                 \max(\frac{\mu^{(k-1)}\cdot s_i-1}{\mu^{(k-1)}},0), & \mbox{if } \ s_i\leq \frac{\mu^{(k-1)}+1}{\mu^{(k-1)}} \\
                 \frac{\mu^{(k-1)}\cdot s_i-\frac{\gamma}{\gamma-1}}
                 {\mu^{(k-1)}-\frac{1}{\gamma-1}}, & \mbox{if } \frac{\mu^{(k-1)}+1}{\mu^{(k-1)}}<s_i<\gamma \\
                 s_i, & \mbox{otherwise}.
               \end{cases}
\end{equation}
We have
\begin{equation*}
\begin{aligned}
 &\mathop{\lim}_{k\to\infty}\|\ML^{(k+1)}- \ML^{(k)}\|_F \\
 & = \mathop{\lim}_{k\to\infty}\|\MY-\MS^{(k+1)} + \frac{1}{\mu^{(k)}}(\MATRIX{A}^{(k)} - \MATRIX{A}^{(k+1)})- \ML^{(k)} + \frac{1}{\mu^{(k-1)}}\MATRIX{A}^{(k-1)}-\frac{1}{\mu^{(k-1)}}\MATRIX{A}^{(k-1)}+\MS^{(k)}-\MS^{(k)}\|_F \\
 & = \mathop{\lim}_{k\to\infty}\|\MY-\MS^{(k)}  + \frac{1}{\mu^{(k-1)}}\MATRIX{A}^{(k-1)}- \ML^{(k)}+\frac{1}{\mu^{(k)}}(\MATRIX{A}^{(k)} - \MATRIX{A}^{(k+1)}) -\frac{1}{\mu^{(k-1)}}\MATRIX{A}^{(k-1)}+\MS^{(k)}-\MS^{(k+1)}\|_F \\
 & \leq \mathop{\lim}_{k\to\infty}\|\MY-\MS^{(k)}  + \frac{1}{\mu^{(k-1)}}\MATRIX{A}^{(k-1)}- \ML^{(k)}\|_F+\|\frac{1}{\mu^{(k)}}(\MATRIX{A}^{(k)} - \MATRIX{A}^{(k+1)}) -\frac{1}{\mu^{(k-1)}}\MATRIX{A}^{(k-1)}\|_F+\|\MS^{(k)}-\MS^{(k+1)}\|_F \\
 & = \mathop{\lim}_{k\to\infty}\|\MY-\MS^{(k)}  + \frac{1}{\mu^{(k-1)}}\MATRIX{A}^{(k-1)}- \ML^{(k)}\|_F \\
 & = \mathop{\lim}_{k\to\infty}\|\MATRIX{U}\MATRIX{S}\MATRIX{V}^{T}- \MATRIX{U}\MATRIX{\Sigma}\MATRIX{V}^{T}\|_F \\
 & = \mathop{\lim}_{k\to\infty}\| \MATRIX{\Sigma}-\MATRIX{S}\|_F \\
 & = \mathop{\lim}_{k\to\infty}\sqrt{\sum_{s_i\leq 1+\frac{1}{\mu^{(k-1)}}}(\frac{1}{\mu^{(k-1)}})^2 + \sum_{\frac{\mu^{(k-1)}+1}{\mu^{(k-1)}}<s_i<\gamma}(\frac{1}{\mu^{(k-1)}})^2} \\
 & = 0.
\end{aligned}
\end{equation*}
The computation in the last limit is analog to equation \eqref{eq-Akplus1}.
\end{proof}
\end{document}